\documentclass{article}
\usepackage{kr}

\usepackage{times}
\usepackage{soul}
\usepackage{url}
\usepackage[hidelinks]{hyperref}
\usepackage[utf8]{inputenc}
\usepackage[small]{caption}
\usepackage{graphicx}
\usepackage{amsmath}
\usepackage{amsthm}
\usepackage{booktabs}
\usepackage{algorithm}
\usepackage[noend]{algpseudocode}
\algrenewcommand\algorithmicindent{0.7em}

\newtheorem{example}{Example}
\newtheorem{theorem}{Theorem}

\usepackage{amssymb}
\usepackage{amsfonts}
\usepackage{modelops} % by Alexander W. Kocurek
\usepackage{tikz}
\usetikzlibrary{arrows.meta, positioning}
\usepackage{paralist}
\usepackage{xspace}
\usepackage{comment}

\newtheorem{definition}{Definition}

\newtheorem{corollary}{Corollary}

\newtheorem{proposition}{Proposition} 

\everymath{\it}\everydisplay{\it}

\newenvironment{mysplit}%
  {\begin{array}{l}}%
  {\end{array}}
  
\newenvironment{logicrule}%
  {\begin{equation}\begin{array}{l}}
  {\end{array}\end{equation}}

\makeatletter
\newcommand\customsize{\@setfontsize\customsize\@ixpt\@viiipt}
\makeatother

\DeclareMathOperator{\val}{{=}}  
 
\DeclareMathOperator{\nval}{{\neq}} 
\DeclareMathOperator{\plus}{{+}}  
\DeclareMathOperator{\minus}{{-}}  

\def\happensAt{\text{\customsize\normalfont\sffamily happensAt}}
\def\initiatedAt{\text{\customsize\normalfont\sffamily initiatedAt}}
\def\terminatedAt{\text{\customsize\normalfont\sffamily terminatedAt}}
\def\holdsAt{\text{\customsize\normalfont\sffamily holdsAt}}
\def\maxDuration{\text{\customsize\normalfont\sffamily fi}}
\def\extensible{\text{\customsize\normalfont\sffamily p}}
\def\true{\text{\customsize\normalfont\sffamily true}}

\def\terminatedIn{\textsf{\customsize termIn}}
\def\initiatedIn{\textsf{\customsize initIn}}

\def\cancelled{\textsf{\customsize cancelled}}

\def\myat{\mathop{\scalebox{0.9}{@}}}
\newcommand{\larsatone}[1]{\myat\nolimits_{#1}}
\newcommand{\larsat}[2]{\myat\nolimits_{#1}\!#2}
\newcommand{\smallboxplus}{%
  \mathbin{\scalebox{0.8}{$\boxplus$}}%
}
\newcommand{\larswindow}[2]{\smallboxplus^{#1}\mskip-2mu#2}
\def\larsbox{\mathord{\scalebox{0.7}{$\Box$}}\mskip2mu}
\def\larsdiamond{\mathord{\scalebox{0.8}{$\Diamond$}}\mskip 2mu}

\newcommand{\tddiamond}[2]{#1^{#2}_{\larsdiamond}}
\newcommand{\tdbox}[2]{#1^{#2}_{\larsbox}}

\def\nbf{\text{\customsize\normalfont\sffamily not}}

\def\ecde{EC\texorpdfstring{\textsuperscript{DE}}{->}\xspace}

\def\datalogones{Datalog$_{1S}$\xspace}
\def\td{Temporal Datalog\xspace}

\def\tdfpneg{Temporal Datalog\texorpdfstring{$^\rightarrow_\circleddash$}{->o}\xspace}

\def\nodelabelling{\mathrm{rule}}
\def\edgelabelling{l}
\newcommand{\labellededge}[1]{\rightarrow_{#1}}

\def\pedges{E^{+}}
\def\nedges{E^{-}}
\def\timelabelling{\mathrm{time}}

\def\pred{\Sigma}

\def\tgmat{\mathsf{TGmat}}
\def\tstgmat{\mathsf{STGmat}}

\def\incrementallystratify{\textsc{IncrementallyStratify}}
\def\incrementallystratify{\textsc{IncrementallyStratify}}
\def\forget{\textsc{Forget}}
\def\larstotd{\mathsf{pLARS2TD}}
\def\ectotd{\mathsf{EC2TD}}
\def\mapecatom{\mathsf{mapEC}}

\def\vecx{\boldsymbol{X}}
\newcommand{\myvec}[1]{\boldsymbol{#1}}

\def\myadd{\mathsf{add}}
\def\myremove{\mathsf{remove}}
\def\myaddall{\mathsf{add\_all}}
\def\perfect{I^*}

\def\replace{replace}

\newcommand{\defeq}{\mathrel{:=}}

\newcommand{\seg}{SEG\xspace}
\newcommand{\segs}{SEGs\xspace}
\newcommand{\stg}{STG\xspace}
\newcommand{\stgs}{STGs\xspace}

\title{Efficient Temporal Datalog Materialisation for Composite Event Recognition}

\author{%
    Periklis Mantenoglou
    \affiliations
    Örebro University, Sweden
    \emails
    periklis.mantenoglou@oru.se
}
\begin{document}

\maketitle

\begin{abstract}
  Several applications demand the timely detection of critical situations, such as threats to safety and transparency, over high-velocity streams of symbolic events.
  This demand has motivated the development of (i) \emph{event specification languages}, which define composite events via temporal patterns over simpler events, and (ii) \emph{stream reasoning frameworks}, evaluating patterns expressed in these languages.
  However, event specification languages are typically studied in isolation, complicating their comparison in terms of expressivity and obscuring the scope of their associated stream reasoners.
  To mitigate this issue, we map practical fragments of prominent event specification languages into \tdfpneg, a temporal Datalog with stratified negation and no future dependencies.
  To support efficient stream reasoning over \tdfpneg, we propose Streaming Trigger Graphs, an extension of a state-of-the-art technique for Datalog materialisation. 
  % and demonstrate that reasoning with TSTGs yields competitive performance to well-established stream reasoners.
  %
  Our approach yields a uniform composite event recognition mechanism that has the potential to generalise across a wide range of practical event specification languages.
  %and achieves high efficiency in practice, with performance competitive with well-established stream reasoners.
  %
  % Note: Maybe we don't even need to go into probabilities in this paper.
  %Such real-world is inherently uncertain, as receiver and transmitter equipment is prone to malfunctions and perturbations in its stability.
  %
  %As a result, there is a need to represent and reason over noisy dynamic  data in real time.
  %
\end{abstract}

\section{Introduction}\label{sec:intro}

Composite event recognition (CER) frameworks detect instances of situations of interest, such as safety and transparency threats, over streams of symbolic events~\cite{DBLP:journals/vldb/GiatrakosAADG20}.
Effective CER is critical in various domains, such as smart cities~\cite{DBLP:journals/simpra/KhazaelVM23}, health systems~\cite{DBLP:journals/artmed/FalcionelliSTMC19}, and fleet management~\cite{DBLP:journals/tplp/TsilionisKNDA19}.
A CER solution typically involves two components: an event specification language (ESL), defining a class of composite event patterns via syntactic constructs and operators, and a reasoner evaluating these patterns with minimal latency over unbounded event streams.
This separation enables a principled study of the trade-offs between ESL expressivity and reasoning tractability, avoiding ad-hoc solutions, and is thus prominent in CER~\cite{DBLP:conf/cikm/DasGZ18,DBLP:journals/pvldb/BucchiGQRV22,tsilionis2024,DBLP:journals/pvldb/AlevizosAP24,DBLP:journals/tplp/WangGWH25,rtecarrow}.

%This separation allows us to formally study and balance ESL expressivity with reasoning tractability, avoiding ad-hoc solutions, and is thus prominent in CER~\cite{DBLP:conf/cikm/DasGZ18,DBLP:journals/pvldb/BucchiGQRV22,DBLP:journals/pvldb/AlevizosAP24,DBLP:journals/tplp/WangGWH25,rtecarrow}.
%
%tractability by extending the ESL towards a more expressive class of patterns, able to express more situations of interest, and then assessing the effects of such an extension to the stream reasoning algorithms, thus avoiding ad-hoc solutions. 
%

Although there are various formal studies on specific ESLs~\cite{DBLP:conf/sigmod/ZhangDI14,DBLP:conf/icdt/GrezRUV20,DBLP:journals/fac/Zielinski23,DBLP:conf/aaai/MantenoglouA25}, a unifying approach is lacking, complicating ESL comparisons and obscuring the scope of the associated stream reasoners.
This also leads to practical issues as composite event patterns expressed in one ESL often cannot be translated into another ESL in a faithful and optimised manner, which may prohibit the use of the reasoner that is best suited to a specific CER problem while also complicating comparisons between reasoners.

Towards addressing this issue, we compile two prominent logic-based ESLs---based on the Event Calculus~\cite{kowalski86} and LARS~\cite{DBLP:journals/ai/BeckDE18}---into a variant of Temporal Datalog~\cite{DBLP:conf/pods/ChomickiI88} with stratified negation and no future dependencies, which we term \tdfpneg.
The Event Calculus has been used in ESLs for large-scale CER~\cite{DBLP:journals/tist/MontaliMCMA13,DBLP:journals/tkde/ArtikisSP15,DBLP:journals/artmed/FalcionelliSTMC19,DBLP:conf/frocos/Baumgartner21,DBLP:conf/kr/MantenoglouKA23,tsilionis2024,rtecarrow}, while LARS features practical fragments for CER~\cite{DBLP:journals/tplp/BeckEB17,DBLP:conf/semweb/BazoobandiBU17}, along with extensions for increased expressivity~\cite{DBLP:conf/ecai/EiterK20,DBLP:conf/kr/UrbaniKE22}, and distribution~\cite{DBLP:journals/tplp/EiterOS19}.

Temporal Datalog has been studied extensively in the context of stream reasoning~\cite{DBLP:conf/pods/ChomickiI88,DBLP:journals/csur/DantsinEGV01,DBLP:conf/aaai/RoncaKGMH18,DBLP:conf/kr/RoncaKGH18,DBLP:journals/ai/RoncaKGH22}, while being closely related to expressive ESLs with dedicated stream reasoners, such as DatalogMTL~\cite{DBLP:conf/aaai/WalegaKG19,DBLP:journals/ws/WalegaKWG23,DBLP:journals/tplp/WangGWH25}.
Thus, \tdfpneg serves as a natural formalism for embedding and studying ESLs.
To support efficient CER over such embedded specifications, we propose Streaming Trigger Graphs (\stgs), an extension of a state-of-the-art materialisation method for Datalog~\cite{DBLP:journals/pvldb/TsamouraCMU21} that supports \tdfpneg programs.
An \stg leverages a so-called \emph{incremental stratification} of the program to enable delta model updates over streaming data and optimisations that reduce redundant computations.

Our \textbf{contributions} may be summarised as follows.
\textbf{First}, we provide mappings from practical Event Calculus and LARS fragments to \tdfpneg, and prove that they are equivalence-preserving with respect to query answering.
\textbf{Second}, we propose \stg-based materialisation as a novel reasoning algorithm for \tdfpneg, demonstrate its correctness and discuss reasoning optimisations.
%
%Third, we propose optimisations for our algorithm and demonstrate that they lead to efficiency gains while preserving correctness.
%
Through the use of \tdfpneg and \stgs, we provide a uniform CER solution for two prominent ESLs, facilitating their formal analysis and comparison.
Our contributions apply to ESLs with stratified negation and no future dependencies, commonly required in stream reasoning~\cite{zaniolo2012}, helping to further consolidate and advance CER.

%Section \ref{sec:background} provides the necessary background concepts.
%
%Section \ref{sec:krr_cer} reviews LARS and the Event Calculus in the context of practical stream reasoning, while Section \ref{sec:mapping} outlines a mapping from practical fragments of these languages into \tdfpneg.
%
%Section \ref{sec:reasoning} proposes an efficient materialisation algorithm for \tdfpneg, which is then empirically proven to be highly-efficient and competitive with the state-of-the-art in Section \ref{sec:experiments}. 
%
%Section \ref{sec:related} compares our work to related approaches from the literature, and Section \ref{sec:summary} summarises our contributions and outlines future research directions.

\section{Background}\label{sec:background}

\subsection{Logic Programming} \label{sec:lp}

We investigate ESLs that are based on logic programming.
Logic programs include variables, constants, predicates, and functors.
Variables start with an upper-case letter, whereas predicates, functors and constants start with a lower-case letter.
A term is either a variable, a constant, or an expression $f(\vecx)$, where $f$ is a functor and $\vecx$ is a term tuple.
An expression $p(\vecx)$, where $p$ is a predicate and $\vecx$ is a term tuple, is called an atom.
A fact is a variable-free atom.
A literal is an atom optionally preceded by ``$\nbf$'', expressing negation-by-failure~\cite{clark78}.
Arity means number of arguments.

A logic program is a set of rules ${h\leftarrow b_1, \dots, b_n}$, where $\leftarrow$ denotes implication, atom $h$ is the head of the rule and its body $b_1, \dots, b_n$ is a conjunction of conditions.
All variables in a rule are implicitly universally quantified.
In definite (resp.~normal) logic programs, all body conditions are atoms (literals).
%
%We silently assume that rules are safe, i.e., every variable in the head of a rule also appears in its body.
%
An expression is called ground if it contains no variables.
A substitution is a mapping of the terms in an expression to constants.
A predicate is called intensional if it appears in the head of a rule with non-empty body, and extensional otherwise. 
A database $D$ is a set of facts with extensional predicates.
%
%$(P, D)$ denotes the program resulting from adding the facts in $D$ in program $P$.
%
A Datalog program is a definite logic program without functors.
Every Datalog program can be rewritten into an equivalent program partitioned into extensional rules (with only extensional body predicates) and intensional rules (with only intensional body predicates). %, by introducing auxiliary predicates where needed.
%
%A rule is called extensional if all predicates in its body are extensional, and intensional otherwise.
%

An interpretation of a logic program $P$ is a set of facts from $P$.
To define the semantics of $P$, we identify its models, i.e., interpretations satisfying all rules in $P$.
Definite logic programs have a unique minimal model, while normal logic programs (which include negation) may have multiple minimal models, complicating their semantic characterisation.
%While programs without negation are typically interpreted based on their unique minimal model, the addition of negation complicates this, as it may lead to multiple minimal models.
%A logic program that includes negation may have multiple minimal models, thus complicating its semantic interpretation.
%A semantics maps each logic program $P$ to its intended meaning.
%
%An interpretation of $P$ is a set of ground atoms from $P$.
% 
%A model of $P$ is an interpretation of $P$ where all rules in $P$ are true statements.
%
%A ground atom $a$ is a logical consequence of $P$ if, for every interpretation $I$, $I$ is a model of $P$ implies that $a\in I$.
%
%The role of a semantics is to identify the model of the program that expresses its intended interpretation, i.e., its intended meaning.
%
%The semantics of a definitive logic program is its unique least model, which coincides with the intersection of all its models~\cite{DBLP:books/sp/Lloyd87}.
%
%On the other hand, normal logic programs may have multiple minimal models, as a result of allowing negation, and thus their intended meaning is not always clear.
%
%Answer Set Programming (ASP), e.g., interprets such programs over the stable model semantics~\cite{DBLP:conf/iclp/GelfondL88}.
%Several semantics have been proposed to address this issue~\cite{DBLP:journals/jlp/AptB94}.
%
To interpret such programs, the closed world assumption (CWA) allows us to infer ``$\nbf\ a$'' if $a$ cannot be finitely proven~\cite{DBLP:conf/adbt/Reiter77}.
Thus, for a program with rule ``$a\leftarrow b, \nbf\ c$'' and fact $b$, model $\{b, a\}$ is preferred over model $\{b, c\}$, because, since it is impossible to prove $c$, the CWA infers ``$\nbf\ c$'', leading to the derivation of $a$ via the rule.
When the program does not include cyclic dependencies via $\nbf$, i.e., it is \emph{stratified}, we can derive an intended interpretation for it by applying the CWA to negated atoms following rule dependencies bottom-up~\cite{DBLP:books/mk/minker88/Przymusinski88}.
%
%
%Two prominent stratification types follow.
%
\begin{definition}[Global Stratification and Local Stratification]\label{def:strat}
Consider a normal logic program $P$ with the predicates in set $\pred$ and the ground atoms in set $B$.
A rule $r$ in $P$ is:
\begin{align}
& \label{eq:lp_rule}
\begin{mysplit}
h(\myvec{X_h})\leftarrow b_1(\myvec{X_{b_1}}), \dots, b_m(\myvec{X_{b_m}}), \\
\qquad\qquad\;\nbf\ c_1(\myvec{X_{c_1}}), \dots, \nbf\ c_k(\myvec{X_{c_k}}).
\end{mysplit}
\end{align}
$P$ is globally stratified if $\pred$ can be decomposed into disjoint sets (strata) $\pred_1$, \dots, $\pred_{n}$, so that, for every $r$ in $P$, we have that, if $h\in \pred_i$, where $i\in [n]$, then: (i) $b_t\in \pred_j$, where $t\in [m]$ and $j\leq i$, and (ii) $c_t\in \pred_j$, where $t\in [k]$ and $j<i$.
Such a decomposition is a global stratification of $P$.

$P$ is locally stratified if $B$ can be split into (possibly infinite) strata $B_1, B_2, \dots$, so that, for every ground instance of a rule $r$ in $P$, if $h(\myvec{x_h})\in B_i$ then: (i) $b_t(\myvec{x_{b_t}})\in B_j$, where $t{\in} [m]$ and $j\leq i$, and (ii) $c_t(\myvec{x_{c_t}})\in B_j$, where $t{\in}[k]$ and $j<i$.
Such a decomposition is a local stratification of $P$.
\end{definition}
%
%According to Definition \ref{def:local_strat}, a program $P$ is locally stratified iff there are no cyclic dependencies with negation in any ground program generated by substituting the variables in $P$ with constants.
%
%For any two ground atoms $a$ and $b$ in a locally stratified program $P$, if $b\priority a$, then $a$ is in a lower stratum than $b$ in every local stratification of $P$.
%
%In order to construct a model of $P$ that reflects the intended meaning of negation, we select a local stratification $\{P_0, \dots, P_n\}$ of $P$ (the model is independent of this choice) 
%
\noindent Based on Definition \ref{def:strat}, every globally stratified program is also locally stratified.
Local stratification, however, is undecidable in the general case, while global stratification can be decided in polynomial time~\cite{DBLP:journals/tcs/Palopoli92}.
A stratified program admits a unique \emph{perfect model}~\cite{DBLP:books/mk/minker88/Przymusinski88}, derived by (i) computing a stratification, and (ii) evaluating its rules in a bottom-up order based on the stratum of their head while applying the CWA to negative conditions.
%We may derive a model for a stratified program $P$ by (i) computing a stratification for $P$, and (ii) evaluating its rules following a bottom-up order based on the stratum of their head while applying the CWA to negative conditions. This yields the unique \emph{perfect model} of $P$~\cite{DBLP:books/mk/minker88/Przymusinski88}.
%
%We implicitly assign this model as the intended meaning of all stratified programs that follow.

\subsection{Temporal Datalog} \label{sec:td}

%We start with preliminaries on temporal-logical languages.
%
A temporal program is a logic program where constants are split into objects and time-points, variables are split into object variables and time variables, and time-points are positive integers.
%
%The time model is linear with positive integer time-points.
%
An object term is an object or an object variable.
A temporal fact is a fact including exactly one time-point.
The model of a temporal program is evaluated over a stream, i.e., an unbounded database of temporal facts.

Datalog programs have been studied extensively~\cite{DBLP:books/cs/Ullman89,DBLP:conf/sigmod/ShkapskyYICCZ16}, and several temporal extensions have been proposed~\cite{zaniolo2012,DBLP:journals/ws/WalegaKWG23,DBLP:journals/tplp/BellomariniBNS25}.
Such an extension is \datalogones where each predicate is augmented with an additional argument featuring a \emph{time term}~\cite{DBLP:conf/pods/ChomickiI88}.
A time term is $T\minus k$, where $T$ is a time variable and $k$ is an integer called the time offset.
Integer addition is silently applied when $T$ is grounded. %in $T\plus k$ when $T$ is ground. % or when $T\plus k$ needs to be compared with another time term.
We focus on \td, i.e., a fragment of \datalogones where each rule may include at most one time variable~\cite{DBLP:journals/ai/RoncaKGH22}.

To formalise \tdfpneg, i.e., the \td variant we use for embedding ESLs, we make two modifications on \td: (i) we introduce a restricted form of negation, and (ii) we disallow derivations towards the past.
A \tdfpneg rule has the following form:
\begin{logicrule}\label{eq:tdfpneg}
h(\myvec{X_h}, T) \leftarrow\\ 
\quad b_1(\myvec{X_{b_1}}, T\minus k_{b_1}), \dots, b_m(\myvec{X_{b_m}}, T\minus k_{b_m}), \\[1pt]
\quad \nbf\ c_{1}(\myvec{X_{c_{1}}}, T\minus k_{c_{1}} ), \dots, \nbf\ c_k(\myvec{X_{c_k}}, T\minus k_{c_k}).
\end{logicrule}
%
%The last argument of every atom in rule \eqref{eq:tdfpneg} is its time argument.

We restrict negation in \tdfpneg by requiring programs to be temporally stratified~\cite{DBLP:journals/tcs/NomikosRG05}.
\begin{definition}[Temporal Stratification]\label{def:tempstrat}
Consider a program $P$ with rules \eqref{eq:tdfpneg}, and the program $P'$ produced by removing all object terms from $P$ and reducing predicate arity to $1$.
$P$ is called temporally stratified if $P'$ is locally stratified.

For a local stratification $B'_1, B'_2, \dots$ of $P'$, the sequence $B_1, B_2, \dots$, where $(p, t)\in B_i$ iff $p(t)\in B'_i$, is a temporal stratification of $P$. 
%A temporal stratification of $P$ is an infinite sequence of strata $\myvec{B}$ such that, for a predicate $p$ and a time-point $t$, $(p, t)$ is in the $i$-th stratum of $\myvec{B}$ if $p(t)$ is in the $i$-th stratum of a local stratification of $P'$.
%
\end{definition}
\noindent According to Definition \ref{def:tempstrat}, if, in a temporally stratified program $P$, $p(\vecx, T)$ depends on $p(\myvec{X'}, T')$ through negation, then ${T'\nval T}$.
Thus, such dependencies are not cyclic in the grounded program, and can be resolved by unfolding (i.e., grounding) time.
As a result, all temporally stratified programs are locally stratified, and thus have a unique perfect model.
%
%Furthermore, since globally stratified programs do not feature cycles with negation regardless of variable grounding, it follows that all globally stratified programs are temporally stratified.
%
Temporal stratification can be decided in polynomial time, while local stratification is co-NP hard on Temporal Datalog with negation~\cite{DBLP:journals/tcs/NomikosRG05}.

To formalise our second modification on \td, we define a forward-propagating program.
\begin{definition}[Forward-Propagating Program]\label{def:fptp}
A program $P$ with rules of the form \eqref{eq:tdfpneg} is forward-propagating if, for every time term $T\minus k$ in the body of a rule in $P$, $k$ is non-negative.
\end{definition}

\begin{definition}[\tdfpneg Program]\label{def:tdfpneg}
A \tdfpneg\ program $P$ is a set of forward-propagating rules of the form \eqref{eq:tdfpneg}, such that $P$ is temporally stratified.
\end{definition}
\begin{example}[\tdfpneg Program]\label{ex:safety}
The following program $P$ evaluates the safety status of a device $X$ over a stream of $repair$ and $warning$ temporal facts.
\begin{align}
& \label{eq:repair}
\begin{mysplit}
safe(X, T)\leftarrow repair(X, T\minus 1). 
\end{mysplit}\\[-2pt]
& \label{eq:notfaulty} 
\begin{mysplit}
safe(X, T)\leftarrow \\
\quad safe(X, T\minus 1), \nbf\ warning(X, T\minus 1). 
\end{mysplit}\\[-2pt]
& \label{eq:trusted}
\begin{mysplit}
    trusted(X, T)\leftarrow \\
   \quad safe(X, T), safe(X, T\minus 1), safe(X, T\minus 2).
\end{mysplit}
\end{align}
Rules \eqref{eq:repair}--\eqref{eq:notfaulty} state that a device $X$ is considered $safe$ at the current time-point $T$ if there was a $repair$ fact with time-stamp $T\minus 1$ or it was proven to be $safe$ at $T\minus 1$ and no $warning$ was produced at $T\minus 1$.
Rule \eqref{eq:trusted} expresses that a device $X$ is considered $trusted$ if it was proven to be $safe$ for three consecutive time-points (including $T$).
%
%Given a stream of $repair$ and $warning$ temporal facts, rules \eqref{eq:repair}--\eqref{eq:trusted} can compute the time-points over which a device is safe and/or trusted, i.e., temporal facts for predicates $safe$ and $trusted$.
%
$P$ is forward-propagating and $B_t\val \{(p, t)|p{\in} Pred(P)\}$ defines a temporal stratification. Thus, $P$ is a \tdfpneg program.
\end{example}
\subsection{Trigger Graphs}\label{sec:tgs}

Trigger Graphs (TGs) were proposed as a technique to steer Datalog materialisation away from redundant computations~\cite{DBLP:journals/pvldb/TsamouraCMU21}, complementing classical semi-naive evaluation and the chase-based approaches~\cite{DBLP:conf/semweb/NenovPMHWB15,DBLP:conf/aaai/UrbaniJK16,DBLP:journals/pvldb/BellomariniSG18}.
Towards defining TGs, we start with Execution Graphs (EGs).
\begin{definition}[Execution Graph]\label{def:eg}
An Execution Graph (EG) for a Datalog program $P$ is an acyclic, node- and edge-labeled digraph $G\val (V, E, \nodelabelling, \edgelabelling)$, where $V$ and $E$ are the set of nodes and the set of edges of the graph, while $\nodelabelling$ and $\edgelabelling$ are the node- and edge-labeling functions. 
$\nodelabelling$ maps each node in $V$ to a rule in $P$.
There can be a labeled edge of the form $u\labellededge{j} v$, i.e., $(u, v)\in E$ and $\edgelabelling((u,v))\val j$, only if the $j$-th predicate in the body of $\nodelabelling(v)$ equals the head predicate of $\nodelabelling(u)$.
%
%In other words, $\edgelabelling$ maps each edge $(u, v)\in E$ to a position in the body of $\nodelabelling(v)$.
%
For each non-leaf node $v$, the EG includes exactly one edge $u \labellededge{j} v$ for each body atom of $\nodelabelling(v)$.
\end{definition}
Each node $v$ of an EG denotes a plan for executing $\nodelabelling(v)$ by specifying the node $u$ that will be used for the evaluation of its $i$-th body condition through an edge $u\labellededge{i} v$.
%
%As a result, EGs encode a plan for materialising a Datalog program. 
%
To enable plan execution, each node $v$ is populated with facts produced by applying $\nodelabelling(v)$ on the facts in the nodes specified by its incoming edges. 
%
%The process of populating the nodes of an EG with facts is formalised as follows.

\begin{definition}[Facts in EG]\label{def:egexec}
Let $P$ be a Datalog program, $D$ a database, $G$ an EG for $P$, $v$ a node in $G$ and $\sigma$ a substitution for the variables in $\nodelabelling(v)$.
The rules in $P$ are partitioned into extensional and intensional, while $h$ and $b$ denote the head and the body of $\nodelabelling(v)$.
$v(D)$, i.e., the set of facts in node $v$ given $D$, contains a fact $h\sigma$ if:
\begin{compactitem}
    \item $\nodelabelling(v)$ is an extensional rule and for each atom $b_i$ in $b$ we have $b_i\sigma\in D$, or
    \item $\nodelabelling(v)$ is an intensional rule and for each atom $b_i$ in $b$ we have $b_i\sigma\in u_i(D)$, where $u\labellededge{i} v$ is an edge in $G$.
\end{compactitem}
$G(D)\val D \cup \bigcup_{v\in V}v(D)$ contains all facts derived by materialising EG $G$ on database $D$.
\end{definition}

If $G(D)$ contains all facts that follow from $P\cup D$, and no other facts, then we say that EG $G$ is a TG.

\begin{definition}[Trigger Graph]\label{def:tg}
Considering a Datalog program $P$, a database $D$, and an EG $G$ for $P$, we say that $G$ is a Trigger Graph (TG) for $P\cup D$ if for each fact $a$ we have $P\cup D\models a$ iff $a\in G(D)$. $G$ is a TG for $P$ if, for each database $D$, $G$ is a TG for $P\cup D$.
\end{definition}

$k$-compatibility determines which nodes of an EG contain facts that may be combined to derive new facts.

\begin{definition}[$k$-compatibility]\label{def:kcompatible}
Let $P$ be a program, $r$ an intensional rule in $P$ and $G$ an EG for $P$.
A tuple $(u_1, \dots, u_n)$ of nodes from $G$ is $k$-compatible with $r$ if:
\begin{compactitem}
\item the predicate in the head of $u_i$ is equal to the predicate in the $i$-th body atom of $r$;
\item the depth of each $u_i$, i.e., the length of the longest path in $G$ that ends at $u_i$, is less than $k$; and
\item at least one node in $(u_1, \dots, u_n)$ has depth $k\minus 1$.
\end{compactitem}
\end{definition}
\begin{algorithm}[t]
\caption{$\tgmat$}
\label{alg:tgmat}
\small
\textbf{Input}: A Datalog program $P$ and a database $D$.\\
\textbf{Output}: The facts in a materialised TG for $P\cup D$.
\begin{algorithmic}[1] %[1] enables line number
\State $k \defeq 0$, $G \defeq \emptyset$, $I^0 \defeq \emptyset$ \label{tgmat:init}
\Repeat
    \State $k \defeq k + 1$, $I^k \defeq I^{k-1}$ \label{tgmat:ruleinit}
    \ForAll{rules $r$ in $P$} \label{tgmat:forrule}
    \If{$r$ is extensional \textbf{and} $\exists \sigma: \forall b_i\in r: b_i\sigma\in D$} \label{tgmat:ifextensional}
    \State $G.\myadd((v, \emptyset, \nodelabelling(v)\val r, \emptyset))$ \label{tgmat:addextensional}
    \Else
    \ForAll{$(u_1, \dots, u_n)$ $k$-compatible with $r$} \label{tgmat:forkcompatible}
    \State $G.\myadd((v, \bigcup_{i\in [n]}(u_i,v), \nodelabelling(v)\val r,$ \label{tgmat:addcompatible}
    \Statex \qquad\qquad\qquad\quad\hspace{4pt}$\bigcup_{i\in [n]}\edgelabelling((u_i, v))\val i))$
    \EndFor
    \EndIf
    \EndFor
    \State $G \defeq \textsc{minDatalog}(G)$ \label{tgmat:minimize}
    \ForAll{nodes $v$ in $G$ of depth $k$}~$I^k.\myadd(v(D, I^{k-1}))$ \label{tgmat:addmat}
    \EndFor
\Until{$I^k \val I^{k-1}$} \label{tgmat:endcond}
\State \Return $I^k$ \label{tgmat:return}
\end{algorithmic}
\end{algorithm}
Algorithm \ref{alg:tgmat} outlines the process of building a TG $G$ for program $P$ and database $D$.
In the first iteration ($k\val 1$), we add one node in $G$ for each extensional rule that fires based on $D$ (lines \ref{tgmat:ifextensional}--\ref{tgmat:addextensional}).
In all subsequent iterations ($k>1$), we use $k$-compatibility to determine whether a node $v$ for rule $r$ and incoming edges from nodes $u_1$, \dots, $u_n$ will be constructed (lines \ref{tgmat:forkcompatible}--\ref{tgmat:addcompatible}). 
The third requirement in Definition \ref{def:kcompatible} ensures that node combinations used in previous iterations are not $k$-compatible, thereby preventing redundant node constructions, in line with semi-naïve evaluation \cite{DBLP:books/cs/Ullman89}.
%,
After each iteration, we remove the nodes of depth $k$ that are subsumed by another node with respect to the query they express (line \ref{tgmat:minimize}), and then populate set $I^k$ with the facts in the remaining nodes of depth $k$ (line \ref{tgmat:addmat}).
$v(D, I^{k\minus 1})$ denotes an optimised materialisation strategy for these facts, restricting attention to tuples that have not been derived at a previous step, i.e., they are not in $I^{k\minus 1}$.
We terminate if no new facts were added in the last iteration (line \ref{tgmat:return}).
%
%See~\cite{DBLP:journals/pvldb/TsamouraCMU21} for more on the optimisations in lines \ref{tgmat:minimize} and \ref{tgmat:addmat}.

\section{Event Specification Languages}\label{sec:krr_cer}

We focus on two prominent logical ESLs based on LARS and the Event Calculus.
A thorough review of ESLs with informative comparisons is provided in Section \ref{sec:related}.

\subsection{LARS}\label{sec:lars}

%LARS is an extension of ASP with temporal operators for stream reasoning~\cite{DBLP:journals/ai/BeckDE18}.
%%
%LARS has also been employed in various settings and tasks, including distributed stream reasoning~\cite{DBLP:journals/tplp/EiterOS19}, preferential, probabilistic and quantitative reasoning~\cite{DBLP:conf/ecai/EiterK20}, and reasoning over rules with existential quantifiers~\cite{DBLP:conf/kr/UrbaniKE22}.
%

LARS extends ASP with window operators and temporal modalities for stream reasoning~\cite{DBLP:journals/ai/BeckDE18}.
A window operator $\larswindow{w} a$ over formula $a$ uses window function $w$ to restrict its evaluation on a particular subset of a stream $S$.
For example, time windows $\larswindow{n}$ restrict attention to the events that took place over the last $n$ time-points. %, while tuple windows $\larswindow{\#n}$ maintain only the last $n$ events.
Temporal modalities express the temporal validity of a formula $a$. 
$\larsdiamond{a}$ (resp.~$\larsbox{a}$) requires that $a$ holds at some (every) time-point in $S$.
$\larsat{T'}{a}$ requires that $a$ holds when the evaluation time is shifted to $T'$, granted that $T'\in S$.
In practice, it is used to express that $a$ holds at some time-point in $S$ (i.e., same as $\larsdiamond$), but also to bind $T'$ to such a time-point.

The main LARS-based stream reasoners are Ticker~\cite{DBLP:journals/tplp/BeckEB17} and Laser~\cite{DBLP:conf/semweb/BazoobandiBU17}. Both are based on the “plain LARS” fragment, whose core building blocks are the so-called \emph{extended atoms}.
\begin{definition}[Extended atom] \label{def:extended_atom}
Let $a$ be an atom $p(\vecx)$, where $p$ is a predicate and $\vecx$ is a sequence of object terms, $T'$ a time variable and $d$ a positive integer.
An extended atom $e$ is defined by the following grammar:
\begin{align*}
e ::= a\ |\ \larsat{T'}{a}\ |\ \larswindow{d}{\larsat{T'}{a}}\ |\ \larswindow{d}{\larsdiamond a}\ |\ \larswindow{d}{\larsbox a}
\end{align*}
%
%We call $p(\vecx)$ the base atom of $e$.
%
\end{definition}
We focus on time windows and rules without $\larsatone{T'}$ operators in the head, leading to forward-propagating programs.
%Definition \ref{def:plain_lars} restricts the use of temporal modalities only over time and tuple windows, which both refer to earlier portions of the stream. 
%
%Thus, plain LARS programs are forward propagating.
%
%We focus our attention on time windows and rules without . 
%
\begin{definition}[Plain LARS]\label{def:plain_lars}
A plain LARS program is a set of rules of the form ``$h \leftarrow b_1, \dots, b_m.$'', where $h$ is an atom, $b_i$, $i\in [m]$, is a (possibly negated) extended atom, and all window operators are time windows. 
\end{definition}

In plain LARS, if a rule $r$ includes condition $\larsat{T'}{a}$, then $T'$ must be grounded earlier in the body of $r$, in order to avoid evaluating $\larsat{T'}{a}$ over arbitrary future time-points.
Therefore, every such condition is preceded by a positive $\larswindow{d}{\larsat{T'}{a}}$ condition, which constrains $T'$ to the window $\{T\minus d, \dots, T\}$, where $T$ is the rule evaluation time.  %with $T$ being the rule evaluation time, restricting the grounding of $T'$.
%
%$\larsat{T'}{p(\vecx)}$ expressions may then be used in the head of a rule to maintain in the output the time-point at which the first $\larswindow{d}{\larsat{T'}{a}}$ body condition was satisfied.

While both Ticker and Laser employ plain LARS, they differ on their support for negation.
Ticker allows non-stratified negation, while Laser is restricted to globally stratified programs.
%
%Regarding reasoning efficiency, Laser outperforms Ticker on stream reasoning over globally stratified programs.
%
We focus on globally stratified plain LARS programs, motivated by the superior reasoning efficiency of Laser~\cite{DBLP:journals/ai/BeckDE18}, as well as by the common assumption in CER that there is a single, implicit model of the world expressing the actual composite event occurrences~\cite{cugola_margara}.
%
%For this reason, we focus on globally stratified plain LARS program, comprising a set of rules of the form \eqref{eq:lars}.

%A globally stratified plain LARS programs yields a unique answer stream~\cite{DBLP:conf/semweb/BazoobandiBU17}.

\begin{example}[Plain LARS Program]\label{ex:lars}
The following plain LARS program monitors safety in a network of devices. 
\begin{align}
& \label{eq:lars_verified}
\begin{mysplit}
verified(X)\leftarrow \larswindow{2}{\larsdiamond repair(X)}.
\end{mysplit}\\[-2pt]
& \label{eq:lars_trusted} 
\begin{mysplit}
trusted(X)\leftarrow \larswindow{6}{\larsbox verified(X)}.
\end{mysplit}\\[-2pt]
& \label{eq:lars_unverified}
\begin{mysplit}
unverified(X)\leftarrow \larswindow{2}{\larsbox \neg repair(X)}.
\end{mysplit}\\[-2pt]
& \label{eq:lars_integrity_threat}
\begin{mysplit}
integrity\_threat\_of(X, Y)\leftarrow \\
   \quad \larswindow{2}{\larsbox trusted(X)}, \larswindow{2}{\larsat{T'}{unverified(Y)}}, \\
   \quad \larswindow{2}{\larsat{T'}{connected(X,Y)}}. 
\end{mysplit}
\end{align}
Rules \eqref{eq:lars_verified}--\eqref{eq:lars_trusted} state that device $X$ is $verified$ if a $repair$ occurred in a two time-step window, and $trusted$ if it has been $verified$ over a six time-step window.
Rule \eqref{eq:lars_unverified} marks $X$ as $unverified$ if no $repair$ occurred in a two time-step window, while rule \eqref{eq:lars_integrity_threat} reports that a device $Y$ poses an integrity threat to a trusted device $X$ if $Y$ was unverified and connected to $X$ at some time-point $T'$ within a two time-step window.
A global stratification places $repair$ in the first stratum and the remaining predicates in the second stratum.
\end{example}

\subsection{Event Calculus with Delayed Effects}\label{sec:ecde}

The Event Calculus (EC) is a temporal logic programming formalism where the functors are classified as event types and \emph{fluent} types, expressing time-varying properties, and a constant sort is reserved for the possible values of each fluent type~\cite{kowalski86}.
An event $e(\vecx)$ (resp.~fluent $f(\vecx)$) is a term where $e$ ($f$) is an event (fluent) type and $\vecx$ contains only objects and variables.
    %
    %A value $v$ for fluent $F$ is an element of set $\rtecvalues^F$.
    %
Expression $f(\vecx)\val v$ states that fluent $f(\vecx)$ has value $v$ and is often used in CER to denote composite events~\cite{DBLP:journals/tkde/ArtikisSP15,DBLP:journals/artmed/FalcionelliSTMC19}.
We focus on an EC variant where events may have \textbf{d}elayed \textbf{e}ffects on fluent values---which we term \ecde---as it is employed in a state-of-the-art CER framework~\cite{rtecarrow}.
%
%Event Calculus variants are employed by numerous composite event recognition and stream reasoning frameworks~\cite{DBLP:conf/frocos/Baumgartner21}.
%
%Among the most prominent ones are: the Run-Time Event Calculus (RTEC), i.e., a logic programming implementation that is optimised for large scale event streams via windowing, caching and indexing~\cite{DBLP:journals/tkde/ArtikisSP15}, and having numerous extensions of RTEC focusing on its expressive power~\cite{DBLP:conf/kr/MantenoglouPA22,DBLP:conf/kr/MantenoglouKA23,rtecarrow} and its reasoning algorithms~\cite{DBLP:journals/jair/TsilionisAP22,tsilionis2024}, and jREC, i.e., an implementation of the ‘Reactive Event Calculus’ that has been extended with indexing for manipulating lists of event and FVP occurrences efficiently~\cite{DBLP:journals/tist/MontaliMCMA13,DBLP:journals/aamas/ChesaniMMT13,DBLP:journals/artmed/FalcionelliSTMC19}.
%
%
%The set of 
%It includes a time sort, as well as sorts for domain objects, events, `fluents', i.e., domain properties whose values may change over time based on event occurrences, and fluent values.
%
The predicates in \ecde and their meaning are as follows.
\begin{compactitem}
\item $\happensAt(e(\vecx), T)$: event $e(\vecx)$ occurs at time $T$.
\item $\holdsAt(f(\vecx)\val v, T)$: fluent $f(\vecx)$ has value $v$ at $T$.
\item $\initiatedAt(f(\vecx)\val v, T)$/$\terminatedAt(f(\vecx)\val v, T)$: a time period where $f(\vecx)\val v$ is initiated/terminated at $T$.
\item $\maxDuration(f(\vecx)\val v, f(\vecx)\val v', d)$: an initiation of $f(\vecx)\val v$ leads to a \textbf{f}uture \textbf{i}nitiation of $f(\vecx)\val v'$ after $d$ time-points, unless $f(\vecx)\val v$ is terminated in the meantime.
\item $\extensible(f(\vecx)\val v)$: the above future initiation of $f(\vecx)\val v'$ is \textbf{p}ostponed by intermediate initiations of $f(\vecx)\val v$.
\end{compactitem}
%
%In other words, a set of events having the immediate effect of initiating $f(\vecx)\val v$ (see rule~\eqref{eq:ec-initiatedAt-schema}) may also have the delayed effect of initiating $f(\vecx)\val v$ after $d$ time-points.
%
%
%In the above expressions, $T$ is a term term or a time-point.

%

In \ecde, fluents are classified as \emph{simple} or \emph{statically determined}.
The values of a simple fluent change based on the satisfaction of initiation and termination conditions and persist via the law of inertia.
%
%The time-points where the FVP holds may then be derived via the law of inertia, expressing that FVPs persist through time unless changed by event occurrences.
%
A statically determined fluent is defined as a Boolean combination of other fluents. % and does not follow the law of inertia.

\noindent\textbf{Simple fluent definitions.}~The rules defining $f(\myvec{X_f})\val v$ based on immediate effects of events have this form:
\begin{logicrule}\label{eq:ec-initiatedAt-schema}
\initiatedAt\mathit{(f(\myvec{X_f})\val v,\ T)} \leftarrow \\
\quad \happensAt(e_1(\myvec{X_{e_1}}),\ T)[,\\ 
\quad [\nbf]\ \happensAt(e_2(\myvec{X_{e_2}}),\ T),\ \dots \\
\quad [\nbf]\ \happensAt(e_n(\myvec{X_{e_n}}),\ T), \\
\quad  [\nbf]\ \holdsAt(f_1(\myvec{X_{f_1}})\val v_1,\ T),\  \dots, \\
\quad [\nbf]\ \holdsAt(f_m(\myvec{X_{f_m}})\val v_m,\ T)]. \\ 
%\quad atemporal\_constraints].
\end{logicrule}
\noindent We use ``$[\ ]$'' to denote optional parts of the rule, i.e., only the first positive $\happensAt$ atom is necessary.
Rules with head $\terminatedAt(f(\myvec{X_f})\val v, T)$ have the same form.  %, expressing the conditions under which $f(\vec{X_f})\val v$ is terminated have the same form.

$\maxDuration$ and $\extensible$ facts express future initiations as follows:
\begin{align}
& \label{eq:delayed-effects-inwindow}
\begin{mysplit}
\initiatedAt\mathit{(F\val V', T{+}d)} \leftarrow \\
	\quad\maxDuration(F\val V, F\val V', d), \initiatedAt(F\val V, T), \\
	\quad\nbf\ \cancelled(F\val V, T, T{+}d). 
\end{mysplit}\\[-2pt]
& \label{eq:delayed-effects-conditions1}
\begin{mysplit}
\mathit{\cancelled(F\val V, T, T\plus d)} \leftarrow \\
	\quad\mathit{\terminatedIn(F\val V, T, T\plus d).}
\end{mysplit}\\[-2pt]
& \label{eq:delayed-effects-conditions2}
\begin{mysplit}
\mathit{\cancelled(F\val V, T, T\plus d)} \leftarrow \\
	\quad\extensible(F\val V), \initiatedIn(F\val V, T, T\plus d).
\end{mysplit} %\tag{\ref{eq:delayed-effects2}$'$}
\end{align}
\noindent $\initiatedIn(F\val V{,} T{,} T\plus d)$ and $\terminatedIn(F\val V{,} T{,} T\plus d)$ express that $F\val V$ is initiated or terminated, respectively, at some time-point strictly between $T$ and $T\plus d$.
In rules \eqref{eq:delayed-effects-inwindow}--\eqref{eq:delayed-effects-conditions2}, $F$ is a variable spanning over fluents, while $V$ and $V'$ span over its possible values.
This abstraction emphasises that these rules are included regardless of the domain being modelled.
However, they may only fire if $\maxDuration$---and possibly $\extensible$---facts are included, expressing the dynamics of delayed fluent value changes in the domain.

\ecde also uses the following rules regardless of domain.
\begin{align}
&
\begin{mysplit}\label{eq:holdsAt1_discrete_inertia}
     \holdsAt(F\val V, T)\leftarrow \initiatedAt(F\val V, T\minus 1).
\end{mysplit}\\[-2pt]
&
\begin{mysplit}\label{eq:holdsAt2_discrete_inertia}
    \holdsAt(F\val V, T) \leftarrow\\
    \quad  \holdsAt(F\val V, T\minus 1),\\
    \quad  \nbf\ \terminatedAt(F\val V, T\minus 1).
\end{mysplit} \\[-2pt]
&
\begin{mysplit}\label{eq:init_diff}
    \terminatedAt(F\val V, T) \leftarrow \\
    \quad  \initiatedAt(F\val V', T), V'\nval V.
\end{mysplit} 
\end{align}
Rules \eqref{eq:holdsAt1_discrete_inertia}--\eqref{eq:holdsAt2_discrete_inertia} express the law of inertia, i.e., $F\val V$ holds at $T$ if it was initiated at the previous time-point $T\minus 1$ or held at $T\minus 1$ and it was not terminated at $T\minus 1$.
%
%This representation of inertia exhibits the Markov property, allowing for optimisations during reasoning, as the initiations, terminations and truth values of $F\val V$ at every time-point before $T\minus 1$ are irrelevant to the evaluation of the truth value of $F\val V$ at $T$.
%
Rule \eqref{eq:init_diff} terminates $F\val V$ when $F$ is initiated with a value other than $V$, restricting fluents to at most one value at a time.

\noindent\textbf{Statically determined fluent definitions.} Such a definition for $f(\myvec{X_f})\val v$ contains $m$ rules where the $i$-th rule is:
\begin{logicrule}\label{eq:ec-holdsAt-schema}
\holdsAt\mathit{(f(\myvec{X_f})\val v,\ T)} \leftarrow \\
\quad  [\nbf]\ \holdsAt(f_{i1}(\myvec{X_{f_{i1}}})\val v_{i1},\ T)[, \\
\quad  [\nbf]\ \holdsAt(f_{i2}(\myvec{X_{f_{i2}}})\val v_{i2},\ T),\ \dots, \\
\quad  [\nbf]\ \holdsAt(f_{in}(\myvec{X_{f_{in}}})\val v_{in},\ T)].
%\quad atemporal\_constraints].
\end{logicrule}
Together, these $m$ rules define $f(\myvec{X_f})\val v$ as a DNF formula on the values held by the fluents in their bodies.
This definition has proven to be equivalent to the interval-based analog used for practical CER~\cite{DBLP:conf/aaai/MantenoglouA25}.
%
%Cyclic dependencies among FVPs via statically determined definitions are forbidden; cycles are allowed if they include only simple FVPs~\cite{DBLP:conf/kr/MantenoglouPA22}.

\ecde\ programs are locally stratified, and thus admit a unique perfect model~\cite{rtecarrow}.
The reasoning task in \ecde\ is to compute $\holdsAt$, i.e., the values of fluents at each time-point.

%Note that, while is possible to have cyclic dependencies on FVPs via simple FVP definitions only~\cite{DBLP:conf/kr/MantenoglouPA22}, such cyclic dependencies via statically determined FVP definitions are forbidden. 

\begin{example}[\ecde Program]\label{ex:ecde}
An \ecde program monitoring device safety may include the following rules.
\begin{align}
& \label{eq:ecde_init_verified}
\begin{mysplit}
\initiatedAt(safety(X)\val verified, T)\leftarrow \\ 
\quad \happensAt(repair(X), T).
\end{mysplit}\\[-2pt]
& \label{eq:ecde_init_unverified}
\begin{mysplit}
\initiatedAt(safety(X)\val unverified, T)\leftarrow \\ 
\quad \happensAt(warning(X), T).
\end{mysplit}\\[-2pt]
& \label{eq:ecde_holds_integrity}
\begin{mysplit}
\holdsAt(integrity\_threat\_of(X,Y)\val\true, T)\leftarrow \\ 
\quad \holdsAt(safety(X)\val trusted, T), \\
\quad \holdsAt(safety(Y)\val unverified, T), \\
\quad \holdsAt(connected(X,Y)\val \true, T).
\end{mysplit}\\[-2pt]
& \label{eq:ecde_fi_trusted}
\begin{mysplit}
\maxDuration(safety(X)\val verified, safety(X)\val trusted, 3).
\end{mysplit}
\end{align}
Rule \eqref{eq:ecde_init_verified} (resp.~rule \eqref{eq:ecde_init_unverified}) states that a device $X$ is considered $verified$ ($unverified$) after a $repair(X)$ ($warning(X)$) event occurs.
Rule \eqref{eq:ecde_holds_integrity} expresses that an $unverified$ device $Y$ poses an integrity threat to a $trusted$ device $X$ if they are connected. 
Fact \eqref{eq:ecde_fi_trusted} expresses that $X$ becomes $trusted$ after being $verified$ for three consecutive time-points.
\end{example}

\section{Embedding ESLs in \tdfpneg}\label{sec:mapping}

%\input{sections/mappings/tdfpneg}

%\subsection{Mapping plain LARS to \tdfpneg}\label{sec:lars2td}
\textbf{Plain LARS to \tdfpneg.}~A plain LARS program $P$ is not \tdfpneg because it includes temporal operators (Definition \ref{def:plain_lars}).
We propose a translation of $P$ into a \tdfpneg program $P'$ that is equivalent with respect to query evaluation for any input stream.
%
%Our goal is to translate $P$ into an equivalent \tdfpneg program $P'$, such that, for every database $D$, a LARS query $Q$ is true over $P$ and $D$ iff its \tdfpneg equivalent $Q'$ is true over $P'$ and $D$.
%
%When discussing formula satisfaction, we often omit the input database for simplicity.
%
%We use $p$ to denote a predicate with arity $n$, $\vecx$ for a tuple of $n$ object terms, $t$ for a time-point, $T'$ for a time variable, $T$ for the rule evaluation time, $d$ for a temporal window size, and $r$ for a rule in $P$.

\begin{algorithm}[t]
\caption{$\larstotd$}
\label{alg:larstotd}
\small
\textbf{Input}: A plain LARS program $P$.\\
\textbf{Output}: A \tdfpneg program $P'$.
\begin{algorithmic}[1] %[1] enables line number
\State $Pred(P')\defeq \emptyset$, $P'\defeq \emptyset$ \label{lars2td:init}
\ForAll{predicates $p\text{/}n$ in $P$}~$Pred(P').\myadd(p\text{/}(n\plus 1))$ \label{lars2td:addpred}
\EndFor

\State $all\_rules\_t\defeq \emptyset$ \label{lars2td:forallrulest}
\ForAll{rules $r\in P$} \label{lars2td:forr}
\State $rules\_t \defeq \{(r, \emptyset)\}$ \label{lars2td:initrulest}
\ForAll{time variables $T'$ appearing in $r$} \label{lars2td:fortvar}
    \State \textbf{let} $\larswindow{d}\larsat{T'}{p}(\vecx)$ \textbf{be} $first\_positive\_at(r, T')$ \label{lars2td:first_pos}
    \ForAll{$(r',\mu)\in rules\_t$} \label{lars2td:forrulestsofar}
    \State $rules\_t.\myremove((r', \mu))$ \label{lars2td:removerulest}
    \ForAll{$i{:} 0{\leq} i{\leq} d$}~$rules\_t.\myadd((r'{,} \mu{\cup} \{T'{\mapsto} T\minus i\}))$ \label{lars2td:addrulest}
    \EndFor
    \State $all\_rules\_t.\myaddall(rules\_t)$ \label{lars2td:addallrulest}
    \EndFor
\EndFor
\EndFor

\ForAll{$(r, \mu)\in all\_rules\_t$} \label{lars2td:forrulesmapped}
\State{$\replace(r, head(r), p_h(\vecx, T))$ \textbf{where} $head(r) \val p_h(\vecx)$} \label{lars2td:headsimple}
%\ElsIf{$h$ \textbf{is} $\larsat{T'} p_h(\vecx)$}\label{lars2td:headat}
%\State $\replace(r{,} h{,} p_h(\vecx{\cup}\mu(T'){,} T))$
%\EndIf
\ForAll{body conditions $b\in r$} \label{lars2td:forbodyb}
\If{$b$ \textbf{is} $p(\vecx)$}~$\replace(r, b, p(\vecx, T))$ \label{lars2td:ifatom}
\ElsIf{$b$ \textbf{is} $\larsat{T'}{p}(\vecx)$}~$\replace(r, b, p(\vecx, \mu(T')))$ \label{lars2td:replaceatsimple}
\ElsIf{$b$ \textbf{is} $\larswindow{d}\larsat{T'}{p}(\vecx)$} \label{lars2td:ifatwindowed}
\If{$T\minus d\leq \mu(T')$}~$\replace(r, b, p(\vecx, \mu(T')))$\label{lars2td:ifatwindowedvalid}
\Else~$\replace(r, b, \bot)$\label{lars2td:ifatwindowedinvalid}
\EndIf

\ElsIf{$b$ \textbf{is} $\larswindow{d}\larsdiamond p(\vecx)$} \label{lars2td:ifdiamond}
    %\If{$\tddiamond{p}{d}\not\in Pred(P')$} \label{lars2td:ifnewdiamond}
    \State $Pred(P').\myadd(\tddiamond{p}{d})$ \label{lars2td:adddiamondpred}
    \ForAll{$i{:} 0{\leq} i{\leq} d$}~$P'.\myadd(\tddiamond{p}{d}(\vecx, T){\leftarrow} p(\vecx, T\minus i))$ \label{lars2td:adddiamondrule}
    \EndFor 
    %\EndIf
    \State $\replace(r, b, \tddiamond{p}{d}(\vecx, T))$ \label{lars2td:replacediamond}

\ElsIf{$b$ \textbf{is} $\larswindow{d}\larsbox p(\vecx)$} \label{lars2td:ifbox}
    %\If{$\tdbox{p}{d}\not\in Pred(P')$} \label{lars2td:ifnewbox}
    \State $Pred(P').\myadd(\tdbox{p}{d})$ \label{lars2td:addboxpredicate}
    \State $P'{.}\myadd(\tdbox{p}{d}(\vecx{,} T){\leftarrow} p(\vecx{,} T){,} \dots{,} p(\vecx{,} T\minus d))$ \label{lars2td:addboxrule}
    %\EndIf
    \State $replace(r, b, \tdbox{p}{d}(\vecx, T))$ \label{lars2td:replacebox}
\EndIf
\EndFor
\State $P'.\myadd(r)$
\EndFor
\State \Return $P'$ \label{lars2td:return}
\end{algorithmic}
\end{algorithm}

Algorithm \ref{alg:larstotd} outlines our translation.
To construct $P'$, we first introduce predicates that denote the atoms in $P$.
For each predicate $p$ with arity $n$ in $P$, we introduce a predicate $p$ with arity $n\plus 1$ in $P'$, where the extra argument is reserved for a time term (see line \ref{lars2td:addpred}).
In order for $P'$ to be \tdfpneg, each rule in $P'$ must contain one time variable $T$, reflecting its evaluation time (see rule schema \eqref{eq:tdfpneg}).
To cater for this, we need to map, in each rule $r$ of $P$, each time variable $T'$ appearing in $\larsatone{T'}$ operators to time terms $T\minus k$.
In plain LARS, $T'$ is grounded based on the first positive $\larswindow{d}\larsat{T'}{p}(\vecx)$ body condition of $r$.
Thus, we replace $r$ with $d\plus 1$ rules such that in the $i$-th rule, where $i\in\{0, \dots, d\}$, $T'$ is mapped to $T\minus i$, denoted via a mapping $\mu$ (see lines \ref{lars2td:forallrulest}--\ref{lars2td:addallrulest}).
This reflects the semantics of $\larsat{T'}{p(\vecx)}$ as an existential quantifier that also maintains via $T'$ the time-points in the window where $p(\vecx)$ holds.

\begin{example}[Plain LARS to \tdfpneg]\label{ex:lars2td}
Consider the plain LARS program in Example \ref{ex:lars}.
Rule \eqref{eq:lars_integrity_threat} contains one time variable $T'$, grounded via $\larswindow{2}{\larsat{T'}{unverified(Y)}}$. 
Thus, Algorithm \ref{alg:larstotd} produces three instances of this rule, with mappings $T'{\mapsto} T\minus 2$, $T'{\mapsto} T\minus 1$ and $T'{\mapsto} T$.
\end{example}

Afterwards, we iterate over each discovered $(r, \mu)$ pair and substitute each atom in $r$ with a \tdfpneg atom (lines \ref{lars2td:forrulesmapped}--\ref{lars2td:replacebox}).
For the head $p_h(\vecx)$ of $r$, we append to $\vecx$ a time variable $T$ denoting the rule evaluation time (line \ref{lars2td:headsimple}).
%
%Otherwise, if $h$ is $\larsat{T'} p_h(\vecx)$, then we also add $\mu(T')$ to the arguments of $p_h$ (line \ref{lars2td:headat})---reified as a concrete object so that $T$ is the only time term of $p_h$---while silently adding the necessary predicates and rules to accommodate this arity increase.
%
Then, we iterate over each body condition $b$ in $r$.
If $b$ contains negation, we implicitly replace ``$\neg$'' with ``$\nbf$'', as their semantics coincide for globally stratified plain LARS programs~\cite{DBLP:conf/semweb/BazoobandiBU17}.
Then, setting negation aside, we distinguish the following cases.
If $b$ is $p(\vecx)$ (resp.~$\larsat{T'}{p}(\vecx)$), then we replace $b$ with $p(\vecx, T)$ ($p(\vecx, \mu(T'))$) in $r$ (lines \ref{lars2td:ifatom}--\ref{lars2td:replaceatsimple}).
%
%If $b$ is $\larsat{T'}{p}(\vecx)$, then we replace $b$ with $p(\vecx, \mu(T'))$ in $r$ (line \ref{lars2td:replaceatsimple}).
%
If $b$ is $\larswindow{d}\larsat{T'}{p}(\vecx)$, then we distinguish two cases.
If $\mu(T')$ falls within $[T\minus d, T]$, then we replace $b$ with $p(\vecx, \mu(T'))$ (line \ref{lars2td:ifatwindowedvalid}).
Otherwise, if $\mu(T'){<}T\minus d$, then we cannot satisfy $b$, and thus disable the rule by replacing $b$ with $\bot$ (line \ref{lars2td:ifatwindowedinvalid}).
If $b$ is $\larswindow{d}\larsdiamond p(\vecx)$ (resp.~$\larswindow{d}\larsbox p(\vecx)$), we introduce predicate $\tddiamond{p}{d}$ ($\tdbox{p}{d}$) with arity $n\plus 1$ and replace $b$ with $\tddiamond{p}{d}(\vecx, T)$ ($\tdbox{p}{d}(\vecx, T)$) (lines \ref{lars2td:ifdiamond}--\ref{lars2td:replacebox}).
Then, we add rules in $P'$ that define $\tddiamond{p}{d}(\vecx, T)$ ($\tdbox{p}{d}(\vecx, T)$) as the disjunction (conjunction) of $p(\vecx{,} T), \dots, p(\vecx{,} T\minus d)$.

\begin{example}[Example \ref{ex:lars2td} cont'd]\label{ex:lars2td_cont}
To translate rule \eqref{eq:lars_verified}, Algorithm \ref{alg:larstotd} replaces its head with $verified(X,T)$, its body with $repair_{\larsdiamond}^{2}(X,T)$ and introduces three rules of the form: ``$repair_{\larsdiamond}^{2}(X,T)\leftarrow repair(X,T\minus i).$'', for $i\in \{0, 1, 2\}$.
Rules \eqref{eq:lars_trusted}--\eqref{eq:lars_unverified} are translated similarly by mapping the $\larswindow{d}\larsbox$ operators they contain according to lines \ref{lars2td:ifbox}--\ref{lars2td:replacebox}.
Rule \eqref{eq:lars_integrity_threat} is translated into three rules, one for each mapping of $T'$ (see Example \ref{ex:lars2td}).
For $T'{\mapsto} T\minus 1$, e.g., we construct this rule:
\begin{align*}
\begin{mysplit}
integrity\_threat\_of(X, Y, T)\leftarrow \\
   \quad trusted_{\larsbox}^{2}(X,T), unverified(Y, T\minus 1), \\
   \quad connected(X, Y, T\minus 1).
\end{mysplit}
\end{align*}
\end{example}

\begin{theorem}[$\larstotd$ Correctness]
Consider a plain LARS program $P$, a stream $S$ and let Algorithm \ref{alg:larstotd} map $P$ to $P'$.
$P'$ is \tdfpneg, and given that $\perfect$ is the perfect model of $P'\cup S$, at time-point $t$, $P\cup S$ derives: 
\begin{compactitem}
\item $p(\vecx)$ iff $\perfect\models p(\vecx, t)$
\item $\larsatone{t'}p(\vecx)$ iff $\perfect\models p(\vecx, t')$
\item $\larswindow{d}\larsatone{t'}p(\vecx)$ iff $\perfect\models p(\vecx, t')$ and $t\minus d\leq t'\leq t$.
\item $\larswindow{d}\larsdiamond p(\vecx)$ iff $\perfect\models \tddiamond{p}{d}(\vecx, t)$
\item $\larswindow{d}\larsbox p(\vecx)$ iff $\perfect\models \tdbox{p}{d}(\vecx, t)$.
\end{compactitem}
\end{theorem}
\begin{proof}[Proof Sketch]
Rules in $P'$ each contain one time variable, are forward-propagating (see, e.g., lines \ref{lars2td:adddiamondrule}, \ref{lars2td:addboxrule} and \ref{lars2td:addrulest}), and $P'$ is temporally stratified since $P$ is globally stratified and Algorithm \ref{alg:larstotd} introduces no negative dependencies.
Thus, $P'$ is \tdfpneg (Definition \ref{def:tdfpneg}).
For model equivalence, we follow an inductive proof on the global stratification $\pred_1, \dots, \pred_n$ of $P$. % which yields a temporal stratification $\myvec{B_1}, \myvec{B_2}, \dots$ for $P'$, where $\myvec{B_t}\val B_t^1 \dots B_t^n$ and $B_t^i\val \{(p, t)| p\in \pred_i\}$, where $i\in [n]$, for each $t\geq 1$.
For extensional predicates, correctness follows from the faithfulness of the translation: e.g., $\larswindow{d}\larsdiamond p(\vecx)$ holds iff $p(\vecx)$ holds now or at least once over the last $d$ time-points, matching the definition of  $\tddiamond{p}{d}$ (line \ref{lars2td:adddiamondrule}).
For the inductive step at $\pred_i$, negative body conditions in any rule defining $p\in\pred_i$ refer to lower strata and are thus evaluated identically by $P$ and $P'$.
Correctness then reduces to the faithfulness of the translation of the temporal operators, which guarantees that the positive conditions are also evaluated equivalently.
\end{proof}

%\subsection{Mapping \ecde\ to \tdfpneg}\label{sec:ecde2td}

\noindent\textbf{\ecde\ to \tdfpneg.}~We consider an \ecde program $P$ consisting of rules \eqref{eq:delayed-effects-inwindow}--\eqref{eq:init_diff}, as well as rules following schemata \eqref{eq:ec-initiatedAt-schema} and \eqref{eq:ec-holdsAt-schema}, and possibly $\maxDuration$ and $\extensible$ facts.
$P$ contains functors, in the form of event and fluent types, atoms with multiple time terms (see rules \eqref{eq:delayed-effects-inwindow}--\eqref{eq:delayed-effects-conditions2}), and an inequality comparison in rule \eqref{eq:init_diff}.
These elements are not in \tdfpneg and thus we need to translate $P$ into a \tdfpneg program $P'$.

\begin{algorithm}[t]
\caption{$\ectotd$}
\label{alg:ectotd}
\small
\textbf{Input}: An \ecde program $P$.\\
\textbf{Output}: A \tdfpneg program $P'$.
\begin{algorithmic}[1] %[1] enables line number
\State $Pred(P')\defeq \emptyset$, $P'\defeq \emptyset$ \label{ec2td:init}
\ForAll{event types $e\text{/}n$ in $P$}~$Pred(P').\myadd(e\text{/}(n\plus 1))$ \label{ec2td:eventtypes}
\EndFor
\ForAll{fluent types $f\text{/}n$ in $P$} \label{ec2td:fluenttypes}
\State $Pred(P').\myadd(f^h\text{/}(n\plus 2))$ \label{ec2td:addfhpred}
\If{$f$ simple}~$Pred(P').\myaddall(f^i\text{/}(n\plus 2){,}f^t\text{/}(n\plus 2))$ \label{ec2td:addsimplepreds}
\EndIf
\EndFor

\ForAll{rules $r\in P$ following schema \eqref{eq:ec-initiatedAt-schema} or schema \eqref{eq:ec-holdsAt-schema}} \label{ec2td:fordomainrule}
    \State $P'.\myadd(\mapecatom(head(r)){\leftarrow} \bigwedge_{b\in body(r)} [\mathsf{not}]\ \mapecatom(b))$ \label{ec2td:adddomainrule}
\EndFor

\ForAll{simple fluent types $f$ \textbf{and} values $v$ of $f$} \label{ec2td:foragnostic}
    \State $P'.\myadd(f^h(\vecx, v,T)\leftarrow f^i(\vecx, v, T\minus 1))$ \label{ec2td:addinertia1}
    \State \hspace{-2pt}$P'.\myadd(f^h(\vecx{,} v{,}T){\leftarrow} f^h(\vecx{,} v{,} T\minus 1),\mathsf{not}\ f^t(\vecx{,} v{,} T\minus 1))$ \label{ec2td:addinertia2}
    \ForAll{values $v'\nval v$ of $f$} \label{ec2td:forvalue}
        \State $P'.\myadd(f^t(\vecx, v,T)\leftarrow f^i(\vecx, v', T))$ \label{ec2td:addinitdiff}
    \EndFor
\EndFor

\ForAll{$\mathsf{fi}(f(\vecx)\val v, f(\vecx)\val v', d)\in P$} \label{ec2td:forfi}
    \State $b\defeq f^i(\vecx, v, T\minus d), \mathsf{not}\ f^t(\vecx, v, T\minus d\plus 1), \dots,$  \label{ec2td:bodybasic}
    \Statex \hspace{88.5pt}$\mathsf{not}\ f^t(\vecx, v, T\minus 1)$ 
    \If{$\mathsf{p}(f(\vecx)\val v)\in P$} \label{ec2td:extensible}
        \State \hspace{-2pt}$b.\myadd(\mathsf{not}\ f^i(\vecx, v, T\minus d\plus 1), \dots,\mathsf{not}\ f^i(\vecx, v, T\minus 1))$ \label{ec2td:addrestbody}
    \EndIf
    \State $P'.\myadd(f^i(\vecx, v',T)\leftarrow b)$ \label{ec2td:addfi}
\EndFor

\State \Return $P'$ \label{ec2td:return}
\Function{$\mapecatom$}{$a$} \label{ec2td:mapfunc}
   \State{\textbf{case} $a$: $\mathsf{happensAt}(e(\vecx), T)$}~\textbf{return} $e(\vecx, T)$ \label{ec2td:maphappens}
    \State{\textbf{case} $a$: $\mathsf{initiatedAt}(f(\vecx)\val v{,} T)$}~\textbf{return} $f^i(\vecx{,} v{,} T)$ \label{ec2td:mapinit}
    \State{\textbf{case} $a$: $\mathsf{terminatedAt}(f(\vecx)\val v{,} T)$}~\textbf{return} $f^t(\vecx{,} v{,} T)$ \label{ec2td:mapterm}
    \State{\textbf{case} $a$: $\mathsf{holdsAt}(f(\vecx)\val v{,} T)$}~\textbf{return} $f^h(\vecx{,} v{,} T)$ \label{ec2td:mapholds}
    %\State{\textbf{case} $a$: atemporal fact $p(\vecx)$}~\textbf{return} $p(\vecx)$ \label{ec2td:atemporal}
    %\EndIf 
   \EndFunction 
\end{algorithmic}
\end{algorithm}

Algorithm \ref{alg:ectotd} outlines the translation steps.
Towards an equivalent program without functors, for each event type $e$ in $P$, we introduce to $P'$ one predicate $e$, adding one more argument to express time (see line \ref{ec2td:eventtypes}).
Similarly, for each fluent type $f$ in $P$ we introduce a predicate $f^h$ in $P'$, denoting when the fluent holds, and equip it with a value argument and a time argument (lines \ref{ec2td:fluenttypes}--\ref{ec2td:addfhpred}).
If $f$ is a simple fluent type, then we additionally introduce predicates $f^i$ and $f^t$ to denote its initiations and terminations (line \ref{ec2td:addsimplepreds}).
As denoted by the $\mapecatom$ function (lines \ref{ec2td:mapfunc}--\ref{ec2td:mapholds}), these new predicates correspond to \ecde predicates for particular events and fluents.

Next, we compile each domain rule, following schema \eqref{eq:ec-initiatedAt-schema} or schema \eqref{eq:ec-holdsAt-schema}, by mapping its head and body atoms to the introduced \tdfpneg predicates (lines \ref{ec2td:fordomainrule}--\ref{ec2td:adddomainrule}). 
Rules \eqref{eq:holdsAt1_discrete_inertia}--\eqref{eq:init_diff} involve fluent variables, while the predicates we introduced are fluent-specific.
To cater for this, we add in $P'$ one such rule for each fluent type $f$ in $P$, while grounding variable $V$ according to the possible values of $f$ (lines \ref{ec2td:foragnostic}--\ref{ec2td:addinitdiff}).
We follow a similar route for rules \eqref{eq:delayed-effects-inwindow}--\eqref{eq:delayed-effects-conditions2}, while additionally compiling away the predicates with two time terms by explicitly adding a body condition referring to each time-point in the range they define (lines \ref{ec2td:forfi}--\ref{ec2td:addfi}). 

\begin{example}[\ecde to \tdfpneg]\label{ex:ecde2td}
Consider the \ecde program in Example \ref{ex:ecde}.
Algorithm \ref{alg:ectotd} translates rules \eqref{eq:ecde_init_verified}--\eqref{eq:ecde_holds_integrity} by mapping their atoms using $\mapecatom$. 
Rule \eqref{eq:ecde_init_verified}, e.g., is translated into: ``$safety^i(X, verified, T){\leftarrow} repair(X, T).$''.
Fact \eqref{eq:ecde_fi_trusted} is expressed in \tdfpneg\ using the following rule:
\begin{align*}
& %\label{eq:ecde_translation_fi_trusted}
\begin{mysplit}
safety^i(X, trusted, T)\leftarrow \\
\quad safety^i(X, verified, T\minus 3), \\
\quad \nbf\ safety^t(X, verified, T\minus 2), \\
\quad \nbf\ safety^t(X, verified, T\minus 1).
\end{mysplit}
\end{align*}
Moreover, for each value of $safety(X)$, we construct rules expressing its default persistence and its incompatibility with other values. 
For value $trusted$, e.g., we have:
\begin{align*}
& %\label{eq:ecde_trans_inertia1}
\begin{mysplit}
safety^h(X, trusted, T)\leftarrow safety^i(X, trusted, T\minus 1). 
\end{mysplit}\\[-2pt]
& %\label{eq:ecde_trans_inertia2}
\begin{mysplit}
safety^h(X, trusted, T)\leftarrow safety^h(X{,} trusted{,} T\minus 1){,} \\
\hspace{113pt}\nbf\ safety^t(X{,} trusted{,} T\minus 1). 
\end{mysplit}\\[-2pt]
& %\label{eq:ecde_trans_diff_val1}
\begin{mysplit}
safety^t(X, trusted, T)\leftarrow safety^i(X, unverified, T).
\end{mysplit}\\[-2pt]
& %\label{eq:ecde_trans_diff_val2}
\begin{mysplit}
safety^t(X, trusted, T)\leftarrow safety^i(X, verified, T).
\end{mysplit}
\end{align*}

\end{example}

\begin{theorem}[$\ectotd$ Correctness]
Consider an \ecde program $P$, a stream $S$ and let Algorithm \ref{alg:ectotd} map $P$ to $P'$.
$P'$ is \tdfpneg, and, for each fluent $f(\vecx)$, value $v$ of $f$ and time-point $t$, $P\cup S$ derives $\holdsAt(f(\vecx)\val v, t)$ iff ${\perfect\models f^h(\vecx, v, t)}$, where $\perfect$ is the perfect model of $P'\cup S$.
\end{theorem}
\begin{proof}[Proof sketch]
All rules added by Algorithm~\ref{alg:ectotd} contain one time variable and are forward-propagating.
Negation appears (i) on conditions with a positive temporal offset (e.g., line \ref{ec2td:addinertia2}), (ii) in a globally stratified statically determined fluent definition, lacking cyclic dependencies~\cite{DBLP:conf/kr/MantenoglouPA22}, and (iii) in an $f^h$ condition of a rule defining $f'^i$ or $f'^t$ ($f\nval f'$).
Case (iii) preserves temporal stratification as $f^h$ may depend negatively on $f'^i$ or $f'^t$ only via a positive time offset. % which does not compromise temporal stratification.
Therefore, $P'$ is \tdfpneg (Definition \ref{def:tdfpneg}).
Semantic equivalence follows from the faithfulness of the translation up to predicate renaming.
%
%
%To prove semantic equivalence, we note that all rules are translated faithfully into \tdfpneg up to the predicate convention changes.
%
\end{proof}

%To address issue (i), we propose Stratified Trigger Graphs (STGs), which extend TGs to support stratified negation.
%
%In the context of stream reasoning, where events arrive incrementally from an unbounded stream, we may attempt to tackle issue (ii) by generating a TG $G^t$ that materialises the
%
%The challenge then lies in handling data that arrive after $G^t$ has been generated.
%
%Recomputing a TG from scratch over all data seen so far is prohibitively expensive, while na\"{\i}vely updating $G^t$ incrementally may lead to incorrect materialisations.
%
%To overcome this limitation, we further extend TGs to support time and enable sound incremental updates, thereby ensuring correct results over all data observed so far.
%

%\subsection{Stratified Negation in Trigger Graphs}\label{sec:gstgs}

%We define Stratified Trigger Graphs (STGs), i.e., an extension of TGs that allows for correct materialisation over globally stratified Datalog programs.
%

\section{Streaming Trigger Graphs}\label{sec:reasoning}
We cannot use TGs to materialise \tdfpneg programs because (i) they do not support negation and (ii) they do not support variables with countably infinite groundings, such as time variables.
The presence of such variables makes the order of rule evaluation crucial, as materialisation can become trapped in an infinite chain of derivations for one predicate, thereby preventing the derivation of facts for other predicates.
Tackling this issue by generating a TG $G^t$ that materialises the program up to the latest time point $t$ is problematic because incrementally updating $G^t$ in the presence of new data may lead to incorrect results, while recomputing a TG from scratch at every time step is prohibitively expensive in the streaming setting.
To address these issues, we propose Streaming Trigger Graphs (\stgs), which constitute an optimised computational model for \tdfpneg.

We start with Streaming Execution Graphs (\segs), which extend EGs with negation and time.
Mirroring the relationship between EGs and TGs, we then define \stgs as \segs that materialise \tdfpneg programs correctly.
\begin{definition}[Streaming Execution Graph]\label{def:tneg}
A Streaming Execution Graph (SEG) for a \tdfpneg program $P$ is a graph $G\val (V,\pedges,\nedges, \nodelabelling, \timelabelling,\edgelabelling)$, where:
\begin{compactitem}
    \item $V$ is a set of nodes.
    \item $\pedges$ is a set of directed edges called p-edges.
    \item $\nedges$ is a set of n-edges, i.e., labelled directed hyperedges $(U, \{v\})$, where $U\subset V$. 
    \item $\nodelabelling$ is a function mapping each node in $V$ to a rule in $P$.
    \item $\timelabelling$ is a mapping of each node in $V$ to a time-point.
    \item $\edgelabelling$ is a function labelling each p-edge $(u, v)$ in $\pedges$ and each n-edge $(U, \{v\})$ in $\nedges$ with a positive integer $j$, denoted respectively as $u\labellededge{j} v$ and ${U\labellededge{j} \{v\}}$.
    We have $u\labellededge{j} v$ (resp.~$u\in U$, where ${U\labellededge{j} \{v\}}$), only if:
    \begin{compactitem}
        \item the predicate of the $j$-th positive (negative) condition of $\nodelabelling(v)$ equals the head predicate of $\nodelabelling(u)$, and
        \item $\timelabelling(u)$ is equal to $\timelabelling(v)$ minus the time offset of the $j$-th positive (negative) body condition in $\nodelabelling(v)$.
    \end{compactitem}
\end{compactitem} 
For each non-leaf node $v$, there is exactly one p-edge ${u {\labellededge{j}} v}$ or n-edge $U{\labellededge{j}} \{v\}$ for the $j$-th condition in $\nodelabelling(v)$.
\end{definition}

Considering a rule $r$ with head $h$, and sets $b^+$ and $b^-$, containing respectively the atoms in positive and in negative body conditions of $r$, if there is a substitution $\sigma$ on the variables in $r$ such that $b^+_i\sigma$ holds for all $b^+_i$ in $b^+$, $r$ derives $h\sigma$ only if $b^-_i\sigma$ does not hold for any $b^-_i$ in $b^-$.  
Intuitively, the atoms in $b^-$ act as filters that disqualify substitutions satisfying at least one such atom.
To reflect this, Definition \ref{def:tneg} uses an n-edge $U\labellededge{i} \{v\}$, where $\nodelabelling(v)\val r$, to denote that if at least one node $u\in U$ contains fact $b^-_i\sigma$ then fact $h\sigma$ should not be introduced in $v$.
Moreover, according to Definition \ref{def:tneg}, \segs restrict body condition evaluation over facts with the appropriate time offset compared to the head predicate.
This minimises the atoms considered in joins during materialisation without compromising correctness.
\begin{definition}[Facts in \seg]\label{def:tnegexec}
Let $P$ be a \tdfpneg program, $S$ a stream, $G$ a \seg for $P$, $v$ a node in $G$ and $\sigma$ a substitution for $\nodelabelling(v)$.
The rules in $P$ are partitioned into extensional and intensional. 
$h$ and $b^+$ (resp.~$b^-$) are the head and the set of atoms in positive (negative) conditions of $\nodelabelling(v)$.
$v(S)$ contains a fact $h\sigma$ if:
\begin{compactitem}
    \item $\nodelabelling(v)$ is extensional and for each atom $b_i$ in $b^+$ ($b^-$) we have $b_i\sigma\in S$ ($b_i\sigma\notin S$), or
    \item $\nodelabelling(v)$ is intensional, for each atom $b_i$ in $b^+$ we have $b_i\sigma\in u(S)$, where $u\labellededge{i} v$ is a p-edge in $G$, and for each atom $b_i$ in $b^-$ and each node $u\in U$, where $U\labellededge{i} \{v\}$ is an n-edge in $G$, we have $b_i\sigma\notin u(S)$.
\end{compactitem}
$G(S)\val S \cup \bigcup_{v\in V}v(S)$ contains all facts derived by materialising \seg $G$ on stream $S$.
\end{definition}

If $G(S)$ contains all facts that follow from $P\cup S$, and no other facts, then we say that \seg $G$ is an \stg.
\begin{definition}[Streaming Trigger Graph]\label{def:tstg}
Considering a \tdfpneg program $P$, a stream $S$ and a \seg $G$ for $P$, we say that $G$ is a Streaming Trigger Graph (\stg) for $P\cup S$ if, for each fact $a$, we have $P\cup S\models a$ iff $a\in G(S)$. $G$ is a \stg for $P$ if, for each stream $S$, $G$ is a \stg for $P\cup S$.
\end{definition}

Towards practical \stg-based materialisation, we introduce a form of temporal stratification that allows us to materialise \tdfpneg programs incrementally in time.
\begin{definition}[Incremental Stratification]\label{def:incrstrat}
Consider a \tdfpneg program $P$, and the program $P_T$ derived by removing from $P$ all body conditions with time offset $k>0$.
A temporal stratification $\myvec{B_1}, \myvec{B_2}, \dots$ of $P$ is called an incremental stratification of $P$ if, $\forall t\geq 1$, all tuples in the strata of $\myvec{B_t}$ have time-point $t$, and the predicate sets obtained from $\myvec{B_t}$ form a global stratification of $P_T$.
\end{definition}
The following corollary holds due to our restriction on forward-propagating and temporally stratified programs.
\begin{corollary}\label{cor:everyincr}
Every \tdfpneg program has an incremental stratification.
\end{corollary}
To construct an incremental stratification $\myvec{B_1}, \myvec{B_2}, \dots$ for program $P$, we compute a global stratification $\pred_1$, \dots, $\pred_n$ for $P_T$, and set $\myvec{B_t}\val B_t^1, \dots, B_t^n$ and $B_t^i\val \{(p, t)| p{\in} \pred_i\}$, where $i\in [n]$, for each $t\geq 1$.
%
%Then, for each $t\geq 1$, we can compute ,. %, while preserving the predicate ordering in $\pred_T$
%
We may then compute the perfect model of $P$ over a stream $S$ by setting $I^0\val S$ and processing $\myvec{B_1}, \myvec{B_2}, \dots$ bottom-up.
For each $\myvec{B_t}$, we substitute $T$ with $t$ in $P$ and evaluate its rules following the predicate stratification in $\myvec{B_t}$, leading to interpretation $I^t$.
%
%The body conditions with time-point $t'<t$ are resolved based only on the temporal facts in $I^{t\minus 1}$. 
%
%The above procedure computes the perfect model of $P$ by applying the CWA incrementally in time. 
%
Crucially, $I^t$ is the perfect model of $P\cup S$ up to time-point $t$, enabling the incremental materialisation of $P$.
%
%We use this stratification schema in our \stg-based materialisation algorithm.  

For TGs, $k$-compatibility ensures that, for each node $v$, its incoming edges stem from nodes defining the predicates in the conditions of $\nodelabelling(v)$ (see Definition \ref{def:kcompatible}).
For STGs, we update $k$-compatibility to ensure that the time labels of these nodes are also consistent with the conditions of $\nodelabelling(v)$.

\begin{definition}[$k$-compatibility]\label{def:kcompatibleseg}
Consider program $P$ with intensional rule $r$ and \seg G.
Tuple $(u_1, \dots, u_n)$ of nodes from $G$ is $k$-compatible with $r$ at time $t$ if:
\begin{compactitem}
\item the predicate in the head of $u_j$ equals the predicate of the $j$-th positive condition of $r$;
\item $\timelabelling(u_j)$ is equal to $t$ minus the offset in the $j$-th positive condition of $r$;
\item each $u_j$ was constructed at an iteration prior to $k$; and 
\item at least one $u_j$ was constructed at iteration $k\minus 1$.
\end{compactitem}
\end{definition}

Unfortunately, $k$-compatibility is inadequate when all tuple nodes were constructed when processing a previous stratum.
Consider, e.g., this traffic light monitoring program:
\begin{align}
    red(L, T)&\leftarrow green(L, T\minus 1). \label{eq:red1} \\
    red(L, T)&\leftarrow green(L, T\minus 2). \label{eq:red2} \\
    green(L, T)&\leftarrow red(L, T\minus 1), red(L, T\minus 2). \label{eq:green}
\end{align}
Given $red(a, 1)$ and $red(a, 2)$, we build an EG $G^4$ up to $t\val 4$.
In the first iteration, we build node for rule \eqref{eq:green} with fact $green(a,3)$.
In the second iteration, we add node for rule \eqref{eq:red1} with fact $red(a,4)$, completing materialisation up to $t\val 4$.
To update $G^4$ beyond $t\val 4$, we need further iterations.
However, since $green(a,3)$ was computed two iterations ago, there are no $k$-compatible nodes with rule \eqref{eq:red2} at $t\val 5$, prohibiting the derivation of $red(a,5)$.
To address this issue, we propose the notion of \emph{consistency}.
\begin{definition}[Consistency]
Consider program $P$ with intensional rule $r$ and \seg G.
Tuple $(u_1, \dots, u_n)$ of nodes from $G$ is consistent with $r$ at time $t$ if:
\begin{compactitem}
\item the predicate in the head of $u_j$ equals the predicate of the $j$-th positive condition of $r$;
\item $\timelabelling(u_j)$ is equal to $t$ minus the offset in the $j$-th positive condition of $r$; and
\item each $u_j$ was constructed at an iteration prior to $k$.
\end{compactitem}
\end{definition}
%
%tuple $(u_1, \dots, u_n)$ of nodes is consistent with $r$ at time $t$ if, $\forall j{\in} [n]$, the predicate in the head of $u_j$ equals the predicate in the $j$-th positive condition of $r$, and $\timelabelling(u_j)$ equals $t$ minus the offset in that condition.
%
%For a rule $r$, the set of consistent node tuples subsumes the set of $k$-compatibility tuples.
%
\begin{algorithm}[t]
\caption{$\tstgmat$}
\label{alg:tstgmat}
\small
\textbf{Input}: A \tdfpneg program $P$ and a stream $S$.\\
\textbf{Output}: The facts in an \stg for $P\cup S$ up to time-point $t$.
\begin{algorithmic}[1] %[1] enables line number
\State $t\defeq 0$, $k \defeq 0$, $S_{\leq 0}\defeq \emptyset$, $G \defeq \emptyset$, $I^0 \defeq \emptyset$ \label{tstgmat:init}
\State $\myvec{B_t} \defeq \incrementallystratify(P)$ \label{tstgmat:stratify}

%\Statex
\While{True} 
\State $t{\defeq} t\plus 1$, $S_{\leq t}{\defeq} S_{\leq t\minus 1} {\cup} read(S, t)$, $P^t{\defeq} ground(P, t)$ \label{tstgmat:init_t} 

\ForAll{$B_t^j\in \myvec{B_t}$} \label{tstgmat:forpredstratum}
\Repeat \label{tstgmat:repeat}
\State $k \defeq k\plus 1$, $I^k\defeq I^{k\minus 1}$

\ForAll{rules $r$ in $P^t$: $pred(head(r))\in B_t^j$} \label{tstgmat:forrule}
\If{$r$ is extensional} \label{tstgmat:ifextensional}
\State $G.\myadd((v, \emptyset, \emptyset, \nodelabelling(v)\val r, \timelabelling(v)\val t, \emptyset))$ \label{tstgmat:addextensional}
\State \textbf{continue} \label{tstgmat:continue}
\EndIf

% set NE 
\State $\nedges\defeq \emptyset$ \label{tstgmat:initne}
\ForAll{$b^-_i\in r$} \label{tstgmat:fornegb}
\State $U\defeq \{u | pred(head(\nodelabelling(u)))\val pred(b^-_i) \wedge $ \label{tstgmat:computeu}
\Statex \qquad\qquad\qquad\quad\;\;$\timelabelling(u)\val t\minus offset(b^-_i)\}$
\State $\nedges.\myadd((U\labellededge{i} \{v\}))$ \label{tstgmat:addne}
\EndFor

% non-recursive case (inter-stratum dependencies only)
\If{$\forall b^+_i\in r$: $pred(b^+_i)\not\in B_t^j \vee offset(b^+_i)>0$} \label{tstgmat:prevstratcond}
\ForAll{$(u_1{,}\dots{,}u_n)$ consistent with $r$} \label{tstgmat:forkapplicable}
\State $G.\myadd((v, \bigcup_{i\in [n]}(u_i,v), \nedges, \nodelabelling(v)\val r,$ \label{tstgmat:addapplicable}
\Statex \qquad\qquad\qquad\quad\hspace{1.4pt}\;\;$ \timelabelling(v)\val t, \bigcup_{i\in [n]}\edgelabelling((u_i, v))\val i))$
\EndFor
% recursive case (at least one intra stratum dependency
\Else~\textbf{for all} $(u_1, \dots, u_n)$ $k$-compatible with $r$ \textbf{do} \label{tstgmat:forkcompatible}
\State $G.\myadd((v, \bigcup_{i\in [n]}(u_i,v), \nedges, \nodelabelling(v)\val r, $ \label{tstgmat:addcompatible}
\Statex \qquad\qquad\qquad\quad\;$\timelabelling(v)\val t, \bigcup_{i\in [n]}\edgelabelling((u_i, v))\val i))$
%\EndFor
\EndIf
\EndFor
\State $G \defeq \textsc{minDatalog}(G)$ \label{tstgmat:minimize}
\ForAll{nodes $v$ of iteration $k$}~$I^k{.}\myadd(v(S_{\leq t}, I^{k\minus 1}))$ \label{tstgmat:addmat}
\EndFor
\Until{$I^k\val I^{k\minus 1}$} \label{tstgmat:until}
\EndFor
\State $output(I^k)$, $\forget(G, S_{\leq t})$ \label{tstgmat:out}
\EndWhile

\end{algorithmic}
\end{algorithm}

\begin{figure*}
\includegraphics[width=.99\linewidth]{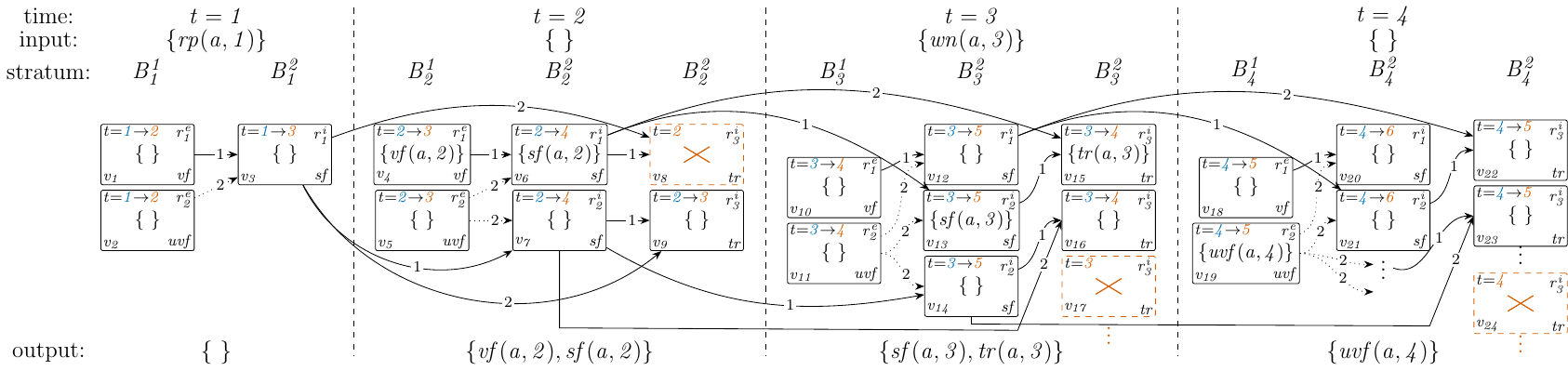}
\caption{STG for Example \ref{ex:tstgmat} up to $t\val 4$. For each node $v$, blue (resp.~vermillion) coloured labels mark $\timelabelling(v)$ (the ``forget time'' of $v$).}
\label{fig:stg}
\end{figure*}

Algorithm \ref{alg:tstgmat} outlines the steps of \stg-based materialisation for program $P$ and stream $S$.
First, we compute an incremental stratification for $P$, with $\myvec{B_t}$ denoting the strata with time-point $t$ (line \ref{tstgmat:stratify}).
Next, for each time-point $t$, we (i) read the facts of $S$ occurring at $t$ (line \ref{tstgmat:init_t}); (ii) ground the time variable of $P$ to $t$, leading to program $P^t$ (line \ref{tstgmat:init_t}); (iii) iterate over the strata $B_t^j$ in $\myvec{B_t}$ and update the \stg $G$ constructed thus far with nodes computing the predicates in $B_t^j$ over $P^t$ (lines \ref{tstgmat:forpredstratum}--\ref{tstgmat:until}); (iv) output the facts computed at $t$ based on $G$ (line \ref{tstgmat:out}); and (v) remove from $G$ and $S_{\leq t}$---the stream read thus far---all nodes and extensional facts that cannot be used for derivations after $t$ according to $P$ (line \ref{tstgmat:out}).

For step (iii), we materialise the predicates in $B_t^j$ in a process similar to TG-materialisation (compare lines \ref{tstgmat:repeat}--\ref{tstgmat:until} and Algorithm \ref{alg:tgmat}).
The first key difference concerns negation: for each negative condition $b_i^-$ in the rule of a node, we add one incoming n-edge whose source is the set of nodes with the same predicate as $b_i^-$ and a time label corresponding to the time-point in $b_i^-$ (lines \ref{tstgmat:initne}--\ref{tstgmat:addne}).
The second difference is that we rely on node consistency---instead of $k$-compatibility---for rules where all body conditions refer to a previous stratum, i.e., their predicate is not in $B_t^j$ or their time-point is earlier than the current one (lines \ref{tstgmat:prevstratcond}--\ref{tstgmat:addapplicable}).

\noindent\textbf{Optimisations.}~TGs reduce redundant computations via two optimisations~\cite{DBLP:journals/pvldb/TsamouraCMU21}: 
(i) after each round $k$, nodes constructed at $k$ whose query is subsumed by another node, based on rewritings recursively replacing intensional predicates with extensional ones following TG dependencies, are removed; and 
(ii) joins during node materialisation are restricted to tuples that are not in the previous interpretation.
Both optimisations extend to \stgs.
For~(i), n-edges stem from nodes in previous strata, so their facts can be treated as extensional in query rewritings (line \ref{tstgmat:minimize}).
Optimisation~(ii) applies unchanged (line \ref{tstgmat:addmat}).
We may further optimise \stgs by \emph{forgetting} nodes whose facts can no longer contribute to future derivations based on their time-stamps.

%As described in step (v), we can ``forget'' \stg nodes that cannot contribute to further derivations based on their time-stamp.
%
\begin{proposition}[Forgetting]
Consider a \tdfpneg program $P$ and an \stg $G$ for $P$ with node $v$ where $\nodelabelling(v)$ defines predicate $p$.
If $k_{p}$ is the maximum temporal offset of any body condition mentioning $p$ across all rules in $P$, then $v$ can be removed from $G$ at evaluation time $t\plus k_{p}\plus 1$.
\end{proposition}
Since $k_{p}$ can be derived offline from $P$, forgetting adds no overhead to online materialisation.

\begin{example}[$\tstgmat$]\label{ex:tstgmat}
Consider the following program $P$, where % where $r$, $w$, $v$, $uv$, $s$ and $tr$ stand for $repair$, $warning$, $verified$, $unverified$, $safe$ and $trusted$.
$rp$, $wn$, $vf$, $uvf$, $sf$ and $tr$ abbreviate $repair$, $warning$, $verified$, $unverified$, $safe$ and $trusted$.
\begin{align}
& \label{eq:v}
\begin{mysplit}
vf(X, T)\leftarrow rp(X, T\minus 1). 
\end{mysplit}
\tag{$r^e_1$}\\[-2pt]
& \label{eq:uv} 
\begin{mysplit}
uvf(X, T)\leftarrow wn(X, T\minus 1). 
\end{mysplit}
\tag{$r^e_2$}\\[-2pt]
& \label{eq:snow}
\begin{mysplit}
    sf(X, T)\leftarrow vf(X, T), \nbf\ uvf(X, T).
\end{mysplit}
\tag{$r^i_1$}\\[-2pt]
& \label{eq:sinertia}
\begin{mysplit}
    sf(X, T)\leftarrow sf(X, T\minus 1), \nbf\ uvf(X, T).
\end{mysplit}
\tag{$r^i_2$}\\[-2pt]
& \label{eq:tr}
\begin{mysplit}
    tr(X, T)\leftarrow sf(X, T), sf(X, T\minus 1).
\end{mysplit}
\tag{$r^i_3$}
\end{align}
$r^e_1$ and $r^e_2$ are extensional, while $r^i_1$, $r^i_2$ and $r^i_3$ are intensional.
Strata $B_t^1\val \{(rp,t),(wn,t),(vf,t),(uvf,t)\}$ and $B_t^2\val \{(sf,t),(tr,t)\}$ incrementally stratify $P$.

Figure \ref{fig:stg} presents the STG constructed by Algorithm \ref{alg:tstgmat} over a stream with $rp(a,1)$ and $wn(a,3)$ up to $t\val 4$.
Starting from $t\val 1$ and $B_1^1$, we first construct nodes $v_1$ and $v_2$ for extensional rules $r^e_1$ and $r^e_2$, defining $vf$ and $uvf$.
These rules cannot fire at $t\val 1$, since their conditions refer to the previous time-point, and thus $v_1$ and $v_2$ are empty.
Moving to the next stratum $B_1^2$, $v_1$ is consistent with intensional rule $r^i_1$---defining $sf$---at time $t\val 1$, leading to the creation of node $v_3$, with incoming p-edge $v_1\labellededge{1} v_3$ and n-edge $\{v_2\}\labellededge{2} v_3$.
No further nodes can be constructed for rules defining predicates in $B_1^2$,  completing materialisation for $t\val 1$.
All conditions in $P$ for predicates $vf$ and $uvf$ have temporal offset $0$, while $sf$ appears in conditions with offset $1$.
Thus, $v_1$ and $v_2$ are forgotten at the next time-step $t\val 2$, while $v_3$ is forgotten when $t\val 3$.

For $t\val 2$ and $B_2^1$, $r^e_1$ fires because of fact $rp(a,1)$, populating node $v_4$ with $vf(a,2)$.
Since  $wn(a,1)$ is not in the input, $uvf(a,2)$ cannot be proven, and thus an empty node $v_5$ is constructed for $r^e_2$.
Moving to $B_2^2$, there are two distinct proofs for $sf(a,2)$: the former requires $vf(a,2)$ via $r^i_1$, while the latter requires $sf(a,1)$ via $r^i_2$.
Node $v_6$ captures the former proof via p-edge $v_4\labellededge{1} v_6$, and is populated with $sf(a,2)$, while $v_7$ captures the latter proof.
Next, we construct nodes $v_8$ and $v_9$, denoting the possible ways of proving $tr(a,2)$.
These proofs, however, are not distinct---given a proof of $sf(a,1)$ via $r^i_1$, reducing $sf(a,2)$ via $r^i_1$ yields a subquery of the query produced through its reduction with $r^i_2$---and thus $v_8$ is removed before materialisation.

At $t\val 3$, $sf(a,3)$ is derived via $r^i_2$, and $tr(a,3)$ then follows from $r^i_3$.
At $t\val 4$, $wn(a,3)$ leads to $uvf(a,4)$ in node $v_{19}$.
The n-edge from $v_{19}$ then blocks the derivation of $sf(a,4)$ via both $r^i_1$ and $r^i_2$, capturing their negative condition. %, demonstrating how negation prevents fact propagation in the STG.
The full trace up to $t\val 4$ is shown in Figure \ref{fig:stg}.
%
%The continuation of $\tstgmat$ for increasing $t$ is similar; Figure \ref{fig:stg} demonstrated it up to $t\val 4$.
%
\end{example}

\begin{theorem}[$\tstgmat$ Correctness]
Algorithm \ref{alg:tstgmat} builds an \stg $G$ for \tdfpneg program $P$ and stream $S$.
\end{theorem}
\begin{proof}[Proof Sketch]
We show by induction on the strata in $\myvec{B_t}$ that $P\cup S\models a(\myvec{x}, t)$ iff $a(\myvec{x}, t)\in G(S)$.
Correctness for $B_1^1$ follows from the correctness of $\tgmat$ (Algorithm \ref{alg:tgmat}), because the predicates in $B_1^1$ have no negative dependencies and their temporal dependencies are vacuous at $t\val 1$.
Assuming correctness up to $B_1^{i\minus 1}$, we prove it for $B_1^i$:
every negated condition ``$\nbf\ b$'' in a rule $r$ defining a predicate $a$ where $(a,1)\in B^i_1$ refers to a lower stratum, so $P\cup S\models b$ iff $b\in G(S)$ by the inductive assumption, while lines~\ref{tstgmat:fornegb}--\ref{tstgmat:addne} add n-edges from all nodes that may contain $b$, ensuring correct evaluation.
Assuming correctness up to $B_{t-1}^n$, correctness for $B_t^1$ follows from the consistency criterion (lines~\ref{tstgmat:prevstratcond}--\ref{tstgmat:addapplicable}), which correctly materialises rules whose conditions all refer to previous strata.
The same arguments support the step from $B^{i\minus 1}_t$ to $B^{i}_t$, completing the proof.
\end{proof}

\section{Summary, Related and Further Work}\label{sec:related}

We proposed mappings from two prominent event specification languages (ESLs) for composite event recognition (CER), which are fragments of LARS and EC, to \tdfpneg, i.e., a temporal Datalog with stratified negation and no future dependencies, and proved their correctness.
We also introduced Streaming Trigger Graphs as an optimised evaluation mechanism for \tdfpneg.
%
%Our contributions aim to unify CER, as their scope covers ESLs that are stratified and forward-propagating---essential properties for CER.
%
%In the future, we aim to implement our approach and compare it with state-of-the-art CER frameworks.

Numerous ESLs have been proposed for CER~\cite{DBLP:conf/debs/CugolaM10,DBLP:conf/frocos/Baumgartner21,DBLP:journals/pvldb/AlevizosAP24}.
DatalogMTL, e.g., combines Datalog with quantitative temporal modalities in a streaming setting~\cite{DBLP:journals/tplp/WangGWH25}.
Contrary to LARS and EC, DatalogMTL is defined over a rational time model and uses intervals to quantify temporal offsets. 
%
%Extending our method for this expressivity is an interesting future direction.
%
DatalogMTL variants with forward-propagating rules~\cite{DBLP:conf/aaai/WalegaKG19}, negation~\cite{DBLP:conf/aaai/CucalaWGK21} and an integer time model~\cite{DBLP:conf/kr/WalegaGKK20} have also been studied.
We conjecture that such a restricted variant can be mapped to \tdfpneg in a similar way to our plain LARS translation, due to similarities between their temporal operators.
Another prominent CER system is CORE~\cite{DBLP:conf/icdt/GrezRUV20,DBLP:journals/pvldb/BucchiGQRV22}.
Contrary to \tdfpneg, CORE is restricted to unary predicates, and thus cannot express relational composite events.

A LARS-to-ASP translation has been proposed for an incremental evaluation of LARS queries~\cite{DBLP:journals/tplp/BeckEB17}, but introduces multiple time variables per rule, producing programs outside \tdfpneg.
Furthermore, the stratification in~\cite{DBLP:conf/ijcai/BeckDE15} prohibits recursion via window operators in LARS, facilitating efficient reasoning.
Our approach does not make this restriction.

Our work is motivated by earlier attempts to avoid ``blocking queries'', requiring arbitrary wait times on new streaming data~\cite{DBLP:conf/pods/BabcockBDMW02}.
\tdfpneg\ prohibits such queries by construction. 
Streamlog is also a temporal Datalog without blocking queries. Unlike our work, it allows equality constraints on time terms and reasons only over time-points appearing in the input~\cite{DBLP:conf/cikm/DasGZ18}.
%
%Streamlog is also a temporal Datalog without blocking queries~\cite{zaniolo2012,DBLP:conf/cikm/DasGZ18}.
%
%It differs from our work as it allows equality constraints on time terms and restricts reasoning to time-points appearing in input facts.
%
%The associated 
%ASTRO ~\cite{DBLP:conf/cikm/DasGZ18} is a stream reasoning framework that implements Streamlog.
%
%Unfortunately, ASTRO was not available to us for experimentation, and thus was not included in our comparisons.

%Real-world event streams and patterns commonly exhibit uncertainty, necessitating probabilistic CER solutions~\cite{DBLP:journals/csur/AlevizosSAP17,Tiger2020,DAsaro2020,tsilionis2025tensor}.
%
%Leveraging a recent trigger graph extension for probabilistic reasoning~\cite{DBLP:journals/pacmmod/TsamouraLU23}, we aim to integrate such reasoning into our approach. 

Several large-scale Datalog engines exist~\cite{DBLP:conf/sigmod/ShkapskyYICCZ16,ryzhyk2019differential,DBLP:conf/kr/Ivliev0MSK24}, and incremental evaluation has been studied for streaming data~\cite{DBLP:journals/sigmod/BudiuCMRT24}, but these approaches do not model temporal modalities.
%
%There are various large-scale Datalog engines~\cite{DBLP:conf/sigmod/ShkapskyYICCZ16,ryzhyk2019differential,DBLP:conf/kr/Ivliev0MSK24}.
%
%Kaminski et al.~(\citeyear{DBLP:journals/jacm/KaminskiKGMH22}) studied the computational properties of Datalog with stratified negation and arithmetic functions over integers.
%
%Incremental Datalog evaluation has been analysed for continuous query processing over streaming data~\cite{DBLP:journals/sigmod/BudiuCMRT24}.
%
%Contrary to our work, these approaches do not model temporal modalities.
%
Temporal Vadalog supports temporal knowledge graph reasoning over existential rules~\cite{DBLP:journals/tplp/BellomariniBNS25}.
Such rules have also been studied in the context of LARS~\cite{DBLP:conf/kr/UrbaniKE22}, DatalogMTL~\cite{DBLP:conf/ijcai/LanzingerNSW23}, and trigger graphs~\cite{DBLP:journals/pvldb/TsamouraCMU21}, thus constituting a natural future extension for \stgs.
Our work assumes ordered input, as is common in stream reasoning~\cite{zaniolo2012}.
Handling out-of-order streams~\cite{DBLP:journals/jair/TsilionisAP22,DBLP:journals/amai/CruzFilipeGN25} via temporal delays in reasoning~\cite{DBLP:journals/ai/RoncaKGH22} is left for future work.
Finally, we aim to leverage a probabilistic extension of trigger graphs~\cite{DBLP:journals/pacmmod/TsamouraLU23} to support CER under uncertainty~\cite{DBLP:journals/csur/AlevizosSAP17,Tiger2020,tsilionis2025tensor}.

%Our work assumes that input events are ordered based on their time-stamps, which is a common practice in stream reasoning~\cite{zaniolo2012}, while some works tackle out-of-order streams~\cite{DBLP:journals/jair/TsilionisAP22,DBLP:journals/amai/CruzFilipeGN25}.
%
%Such streams would introduce a temporal delay to \tdfpneg materialisation~\cite{DBLP:journals/ai/RoncaKGH22}; formalising such delays would be interesting for future work.

%\input{sections/summary}

\section*{AI Declaration}

No generative AI tools were used in the research, development, or writing of this paper.

\section*{Acknowledgements}

This work was supported by the Wallenberg AI, Autonomous Systems and Software Program (WASP) funded by the Knut and Alice Wallenberg Foundation.

%% The file kr.bst is a bibliography style file for BibTeX 0.99c
\bibliographystyle{kr}
\bibliography{refs}

%\appendix

%\section*{Appendix}

%\input{appendix}
%\input{sections/prob/motivation}

%\input{sections/prob/probabilistic-databases}

%\input{sections/prob/tp_compilation}

%\input{sections/prob/trigger_graphs}

\end{document}